\PassOptionsToPackage{dvipsnames}{xcolor}
\documentclass{article}

\usepackage[preprint]{neurips_2026}

\usepackage[utf8]{inputenc} 
\usepackage[T1]{fontenc}    
\usepackage{hyperref}       
\usepackage{url}            
\usepackage{booktabs}       
\usepackage{amsfonts}       
\usepackage{nicefrac}       
\usepackage{silence}
\usepackage{microtype}

\usepackage{amsmath}
\usepackage{amssymb}
\usepackage{amsthm}
\usepackage{graphicx}
\usepackage{subcaption}
\usepackage{enumitem}
\usepackage{cleveref}
\usepackage{tikz}
\usetikzlibrary{positioning,calc}
\usepackage{array}
\usepackage{makecell}
\usepackage{tabularx}

\tikzset{
  stagebox/.style={draw, rounded corners=1pt, minimum height=0.42cm, inner sep=1pt},
  fwd/.style={stagebox, fill=blue!18},
  bwd/.style={stagebox, fill=orange!28},
  syncbar/.style={draw=gray!70, fill=gray!22},
  workbar/.style={draw=teal!60!black, fill=teal!20}
}

\usepackage{multirow}

\usepackage[T1]{fontenc}
\usepackage{calligra}
\usepackage{layout}

\usepackage{amsmath}
\usepackage{amssymb}
\usepackage{amsfonts}
\usepackage{amsthm}

\usepackage{mathtools}
\usepackage{thmtools} 
\usepackage{thm-restate}

\usepackage{latexsym}

\usepackage[dvipsnames]{xcolor}
\usepackage{graphicx}
\usepackage{verbatim}

\usepackage{algorithmic}
\usepackage{algorithm}

\usepackage{url}
\usepackage{hyperref}
\hypersetup{colorlinks=false, urlcolor=magenta}

\usepackage{cleveref}

\usepackage{bbm}
\usepackage{bm} 

\usepackage{nicefrac}
\usepackage{adjustbox}
\usepackage{threeparttable}

\usepackage{import, ifthen}

\usepackage{tcolorbox}
\usepackage{pifont}
\definecolor{mydarkred}{RGB}{192,25,25}
\definecolor{mydarkgreen}{RGB}{25,192,25}
\definecolor{mydarkblue}{RGB}{25,25,192}
\newcommand{\red}{\color{RedOrange}}

\usepackage{xspace}

\newcommand{\algname}[1]{{\color{gray}\small\sf#1}\xspace}

\newcommand{\norm}[1]{{\left\| #1 \right\|}}

\newcommand{\anglebr}[1]{\left\langle#1\right\rangle} 

\newcommand{\inner}[2]{\anglebr{#1, #2}}

\newcommand{\cO}{\mathcal{O}}

\newcommand{\R}{\mathbb{R}}  

\newcommand{\del}[1]{}

\newcommand{\eqdef}{\coloneqq}

\DeclareMathOperator{\E}{{\bf E}}   

\newcommand{\Exp}[1]{{\red\mathbb{E}} \left[ #1 \right]}

\newcommand{\pr}[1][]{
  \ifthenelse { \equal{#1}{} }
  { \ensuremath{\mathrm{P}} }
  { \ensuremath{\mathrm{P}\left(#1\right)} }
}

\usepackage[colorinlistoftodos,bordercolor=orange,backgroundcolor=orange!20,linecolor=orange,textsize=scriptsize]{todonotes}

\newcommand{\ppeter}[1]{}

\setcitestyle{authoryear,round,citesep={;},aysep={,},yysep={;}} 

\newcommand{\hidesolutions}[1]{} 

\definecolor{thmbackground}{rgb}{240, 240, 255}

\declaretheoremstyle[
  spaceabove = 10pt, 
  spacebelow = 10pt,
  postheadspace=\newline,  
    postheadhook={\textcolor{black}{\rule[.6ex]{\linewidth}{0.4pt}}\\},  
  headfont = \bfseries,
  bodyfont = \normalfont\itshape,
  notefont = \mdseries\bfseries,
  notebraces = (),
 mdframed={
            frametitlebackgroundcolor=blue!60,
            backgroundcolor=red!0!white, 
           linecolor=black!100!white, 
            linewidth=0.4pt,     
            roundcorner=5pt,                     
            innertopmargin=6pt,
            innerbottommargin=20pt,   
            skipabove=2\parsep, 
            skipbelow=2\parsep } 
]{mythmstyle}

\declaretheorem[
  name=Theorem,
  style=mythmstyle,
  numberwithin=section
]{theorem}

\declaretheorem[
  name=Corollary,
  style=mythmstyle,
  numberwithin=section
]{corollary}

\declaretheorem[
  name=Proposition,
  style=mythmstyle,
  numberwithin=section
]{proposition}

\declaretheoremstyle[
  spaceabove = 10pt, 
  spacebelow = 10pt,
  postheadspace=\newline,  
    postheadhook={\textcolor{black}{\rule[.6ex]{\linewidth}{0.4pt}}\\},  
  headfont = \bfseries,
  bodyfont = \normalfont\itshape,
  notefont = \mdseries\bfseries,
  notebraces = (),
 mdframed={
            backgroundcolor=green!0!white, 
            linecolor=black!100!white, 
            innertopmargin=6pt,
            roundcorner=5pt, 
            innerbottommargin=20pt, 
            skipabove=2\parsep, 
            skipbelow=2\parsep } 
]{mylemmastyle}

\declaretheorem[
  name=Lemma,
  style=mylemmastyle,
  numberwithin=section
]{lemma}

\declaretheoremstyle[
  spaceabove = 10pt, 
  spacebelow = 10pt,
  postheadspace=\newline,  
    postheadhook={\textcolor{black}{\rule[.6ex]{\linewidth}{0.4pt}}\\},  
  headfont = \bfseries,
  bodyfont = \normalfont\itshape,
  notefont = \mdseries\bfseries,
  notebraces = (),
 mdframed={
            backgroundcolor=black!0!white, 
            linecolor=black!100!white, 
            innertopmargin=6pt,
            roundcorner=5pt, 
            innerbottommargin=20pt, 
            skipabove=2\parsep, 
            skipbelow=2\parsep } 
]{myexamplestyle}

\theoremstyle{plain}
\newtheorem{assumption}{Assumption}\numberwithin{assumption}{section}
\numberwithin{claim}{section}
\numberwithin{fact}{section}
 \numberwithin{exercise}{section}

\theoremstyle{definition}
\numberwithin{definition}{section}

\title{Demystifying Pipeline Parallelism:\\First Theory for PipeDream}

\author{%
  Ivan Ilin \\
  KAUST \\
  \texttt{ivan.ilin@kaust.edu.sa} \\
  \And
  Peter~Richt\'{a}rik \\
  KAUST \\
  \texttt{peter.richtarik@kaust.edu.sa}
}

\begin{document}

\maketitle

\begin{abstract}
Training modern machine learning models increasingly requires computation to be distributed across many accelerators. Data parallelism remains the default choice and is often paired with tensor-parallel sharding, but model parallelism becomes unavoidable once parameters, activations, or optimizer states no longer fit on a single device. This paper studies pipeline model parallelism through the lens of PipeDream (\algname{PD}) \citep{harlap2018pipedream}. Our first contribution is theoretical: we introduce Randomized PipeDream (\algname{RPD}), a stale block-SGD abstraction that yields, to our knowledge, the first clean nonconvex convergence guarantee for a \algname{PD}-style method. Our second contribution is a scaling diagnosis: we prove that the delay induced by steady-state \algname{PD} grows as $S^2 - \nicefrac{S}{2} + O(1)$ for $S$ stages, so the stale-read contribution in the convergence theorem scales as $\Theta(\gamma^2 S^4)$, equivalently as $\Theta(S^4/K)$ in the tuned-rate form. Our third contribution is a comparison with \algname{LocalSGD}, whose periodic model averaging trades weight staleness for synchronization bubbles. In our reported simulated-time experiments, the better-performing method depends on the objective: \algname{PD} performs better on the quadratic objective and on a small language-modeling training-loss task, while for logistic regression \algname{LocalSGD} becomes superior as the number of stages increases.
\end{abstract}

\section{Introduction}

The recent success of deep learning has been driven in part by scaling: larger models, larger datasets, and longer training runs typically deliver better downstream performance. That scaling pressure immediately creates a systems problem. Even when a model fits on one accelerator, training time can be prohibitive; once the parameters, activations, and optimizer states no longer fit on one accelerator, single-device training is impossible. Distributed optimization is therefore not an implementation detail but a prerequisite for modern large-scale learning \citep{dean2012large,goyal2017accurate}.

The most widely used training strategy is data parallelism: each worker holds a model replica, computes gradients on different mini-batches, and synchronizes updates across workers. In practice, large training stacks often combine this outer-loop data parallelism with tensor-parallel sharding inside layers to spread matrix multiplications across multiple devices \citep{shoeybi2019megatron}. By contrast, \emph{model parallelism} partitions the model itself across devices. This partition can be even or uneven, and it may follow layers, tensor dimensions, or other architectural boundaries. Model parallelism is less ubiquitous than pure data parallelism, but it is equally important whenever the model simply cannot be stored or executed on a single device \citep{huang2019gpipe,narayanan2019pipedream}.

This paper focuses on pipeline model parallelism. The basic idea is simple: split a deep network into $S$ stages, place stage $s$ on device $s$, and stream microbatches through the stage sequence. However, the naive schedule is painfully inefficient. If a model is partitioned across four devices, a single microbatch moves forward from stage $1$ to stage $4$ and then backward from stage $4$ to stage $1$. During much of that time, several machines are idle (\Cref{fig:naive-pipeline}). These idle regions are the classic \emph{pipeline bubbles}. PipeDream (\algname{PD}) \citep{narayanan2019pipedream} was one of the first influential methods to attack this issue: it overlaps multiple microbatches and uses a one-forward-one-backward (1F1B) steady-state schedule so that machines spend far less time doing nothing (\Cref{fig:pipedream-1f1b}).

The systems advantage of \algname{PD} is clear, but its optimization behavior is hard to analyze. The core obstacle is not merely delay; it is \emph{structured}, stage-dependent delay produced by weight stashing and by the deterministic 1F1B schedule. Each backward pass uses a collection of stage-local weight versions encountered earlier during the corresponding forward pass, and these versions are coupled through the pipeline timeline. Existing literature provides excellent systems insight into pipeline-parallel training, yet a tractable convergence theory for \algname{PD} itself remains largely missing.

This paper has three main contributions.
\begin{enumerate}[leftmargin=1.2em,itemsep=0.2em,topsep=0.2em]
    \item We introduce Randomized PipeDream (\algname{RPD}), a mathematical abstraction in which the active stage and mini-batch are sampled uniformly and the stage-wise staleness is arbitrary but bounded by a hyperparameter $\delta$. For this abstraction we prove a clean nonconvex convergence rate, and we then connect the theory back to the original method by showing that steady-state \algname{PD} is described by $\delta = S^2 - \nicefrac{S}{2} + O(1)$.
    \item We use this theory to explain why \algname{PD} scales poorly with the number of stages. Because the theorem contains a $\delta^2$ term and because $\delta = \Theta(S^2)$, the stale-read term grows as $\Theta(\gamma^2S^4)$, or $\Theta(S^4/K)$ after the tuned stepsize substitution. Under a fixed iteration budget, increasing the number of stages slows objective decrease substantially.
    \item We study \algname{LocalSGD} as an alternative schedule that trades PipeDream-style stale weight versions for synchronization barriers and local-replica drift. The resulting comparison is not one-sided: our experiments on quadratic objectives, logistic regression, and an $11.1$M-parameter NanoChat-style Transformer show that the better schedule depends on the objective and scaling regime. \algname{PD} is stronger on the quadratic and language-modeling tasks, while \algname{LocalSGD} outperforms it on logistic regression when the number of stages is large and gradients are noisy.
\end{enumerate}

The implementation and code for reproducing our experiments are available at \href{https://github.com/vectozavr/randomized-pipedream}{\texttt{github.com/vectozavr/randomized-pipedream}}.

\subsection{Notation} All key notation used in this paper is summarized in a tabular form in \Cref{sec:notation}; see Table~\ref{tab:notation_table}.

\section{Related Work}
\label{sec:related-work}

\paragraph{Distributed and pipeline-parallel training.}
Distributed training is central to scaling modern deep networks, either by replicating models across workers in data-parallel SGD
\citep{dean2012large,goyal2017accurate}, or by partitioning computation, parameters, activations, and optimizer states across devices. Tensor-parallel systems such as Megatron-LM split large Transformer layers across accelerators \citep{shoeybi2019megatron}, while \algname{ZeRO} reduces data-parallel memory redundancy by partitioning optimizer states, gradients, and parameters \citep{rajbhandari2020zero}. More general systems such as \algname{FlexFlow} and \algname{Alpa} search over hybrid parallelization plans that combine data, operator, tensor, and pipeline parallelism \citep{jia2019beyond,zheng2022alpa}. Our work is orthogonal to these systems contributions: we isolate the optimization effect of
the stale weights induced by pipeline-parallel execution.

Pipeline parallelism partitions a model into consecutive stages and streams microbatches through them. \algname{GPipe} uses synchronous microbatch pipelining and updates weights only after a mini-batch has been flushed through the pipeline
\citep{huang2019gpipe}. \algname{PD} instead uses inter-batch pipelining with a one-forward-one-backward schedule to keep stages active \citep{harlap2018pipedream,narayanan2019pipedream}. Its weight-stashing mechanism preserves consistency between the forward and backward pass of each microbatch, but gradients are evaluated using stale stage-local weights. Subsequent systems such as \algname{PipeDream-2BW} and \algname{DAPPLE} reduce memory overhead, improve scheduling, or combine pipeline and data parallelism \citep{narayanan2021memory,fan2021dapple}. These works primarily optimize throughput, memory, placement, and scheduling. In contrast, we study the nonconvex optimization cost of the stale weight versions created by PipeDream-style execution.

\paragraph{Stale-gradient and randomized block optimization.}
The stale gradients arising in pipeline parallelism are related to asynchronous optimization. \algname{Hogwild!} established convergence of lock-free stochastic updates under sparsity assumptions \citep{recht2011hogwild}, while later work analyzed asynchronous SGD for nonconvex objectives under bounded delays \citep{lian2015asynchronous}. The perturbed-iterate framework models stale reads as gradients evaluated at perturbed versions of the current iterate \citep{mania2017perturbed}, and delay compensation methods attempt to correct staleness using Taylor or Hessian-based approximations \citep{zheng2017asynchronous}. Our setting differs from generic asynchronous SGD because the delay structure is generated by a deterministic pipeline schedule, stage-local weight stashing, and forward/backward ordering. Our Randomized PipeDream abstraction keeps the stale block-gradient structure while removing the deterministic scheduling combinatorics.

Our abstraction also connects to randomized coordinate and block-coordinate methods, which update only part of the variable at each iteration \citep{nesterov2012efficiency,richtarik2014iteration}. Block stochastic-gradient methods combine data sampling with block-wise updates and have been analyzed in convex and nonconvex settings \citep{xu2015block}. However, unlike standard block-coordinate methods, our active block gradient is evaluated at a stale model assembled from different stage-local weight versions.

\paragraph{Local SGD and model averaging.}  \algname{LocalSGD} reduces communication by letting workers perform several local stochastic-gradient steps before averaging their models. This idea goes back to parallelized SGD with model averaging \citep{zinkevich2010parallelized}; later work showed that \algname{LocalSGD} can match mini-batch SGD rates with fewer communication rounds in convex settings \citep{stich2018local}, and explained its behavior for nonconvex deep learning objectives \citep{yu2019parallel}. Empirical work also suggests that local updates can improve efficiency and generalization compared with simply increasing the mini-batch size \citep{lin2018don}. We use \algname{LocalSGD} as an alternative policy for stage-distributed training: several full-model trajectories are executed through
the same pipeline stages and periodically averaged. This avoids \algname{PD}'s cross-stage weight staleness, but introduces synchronization barriers and local-replica drift, yielding a concrete staleness-versus-synchronization comparison.

\section{Background: Why Pipeline Parallelism?}

We study the finite-sum problem
\begin{equation}
\label{eq:objective}
\min_{w \in \R^d} f(w),
\qquad
f(w) = \frac{1}{M}\sum_{m=1}^{M} f_m(w),
\end{equation}
where the model is partitioned into $S$ blocks
$$w = \bigl(w^{(1)},\dots,w^{(S)}\bigr), \qquad w^{(s)} \in \R^{d_s}, \qquad \sum_{s=1}^{S} d_s = d.$$
Each block corresponds to one pipeline stage. In practice, the partition may be balanced or intentionally uneven to match heterogeneous devices or layer costs.

\paragraph{Naive pipelining.}
If we send a single microbatch through the pipeline, the forward pass proceeds stage-by-stage from $1$ to $S$, and the backward pass returns from $S$ to $1$. This schedule is functionally correct but extremely inefficient because most devices are idle during startup and drain. The result is a large bubble fraction, especially for deep pipelines (\Cref{fig:naive-pipeline}).

\paragraph{PipeDream.}
\algname{PD} reduces these bubbles by admitting multiple microbatches into the pipeline and switching each stage into a 1F1B steady state after startup \citep{narayanan2019pipedream}. The price is that gradients become stale: a backward pass does not use the newest global model, but rather the stage-local weights that were seen earlier during the corresponding forward pass. Weight stashing preserves forward/backward consistency within each stage, but it also creates a complicated pattern of cross-stage version dependencies.

\Cref{fig:naive-vs-pipedream} visualizes the contrast between \algname{PD} and a naive pipeline. The systems motivation for \algname{PD} is therefore immediate: it converts idle pipeline time into useful work. The theoretical question is more subtle: what is the optimization cost of the induced staleness, and how does that cost scale with the number of stages? The discrete-event procedure used to generate the \algname{PD} timelines in our experiments is described in Appendix~\ref{sec:pd-1f1b-schedule}.

\begin{figure}[t]
    \centering
    \begin{subfigure}[t]{0.85\linewidth}
        \centering
        \includegraphics[width=\linewidth]{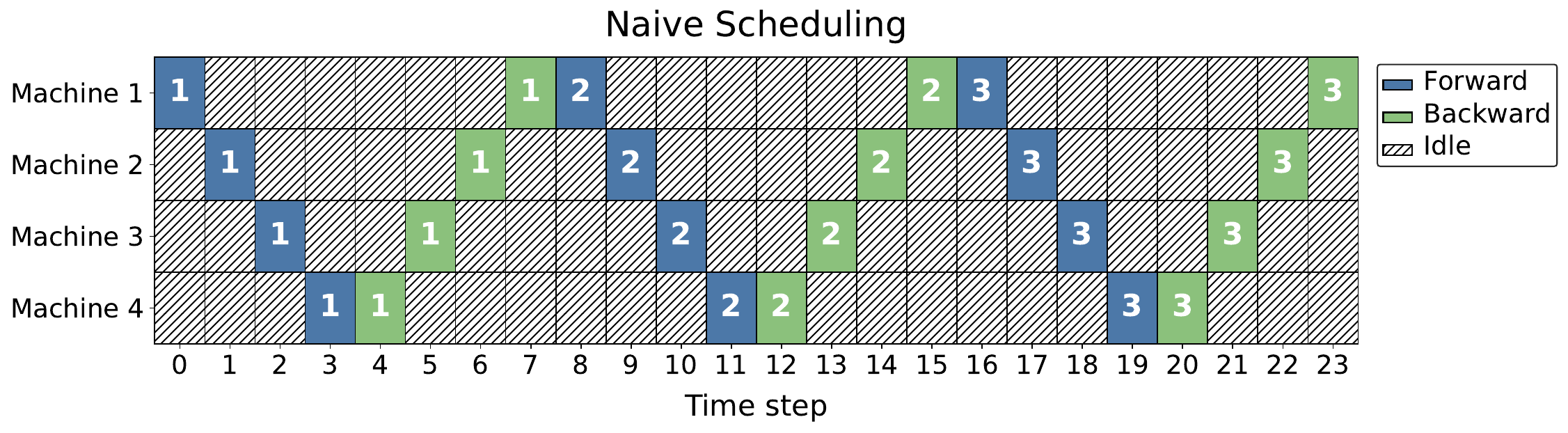}
        \caption{Naive pipeline for one microbatch. The blank area is pure idle time.}
        \label{fig:naive-pipeline}
    \end{subfigure}
    \hfill
    \begin{subfigure}[t]{0.85\linewidth}
        \centering
        \includegraphics[width=\linewidth]{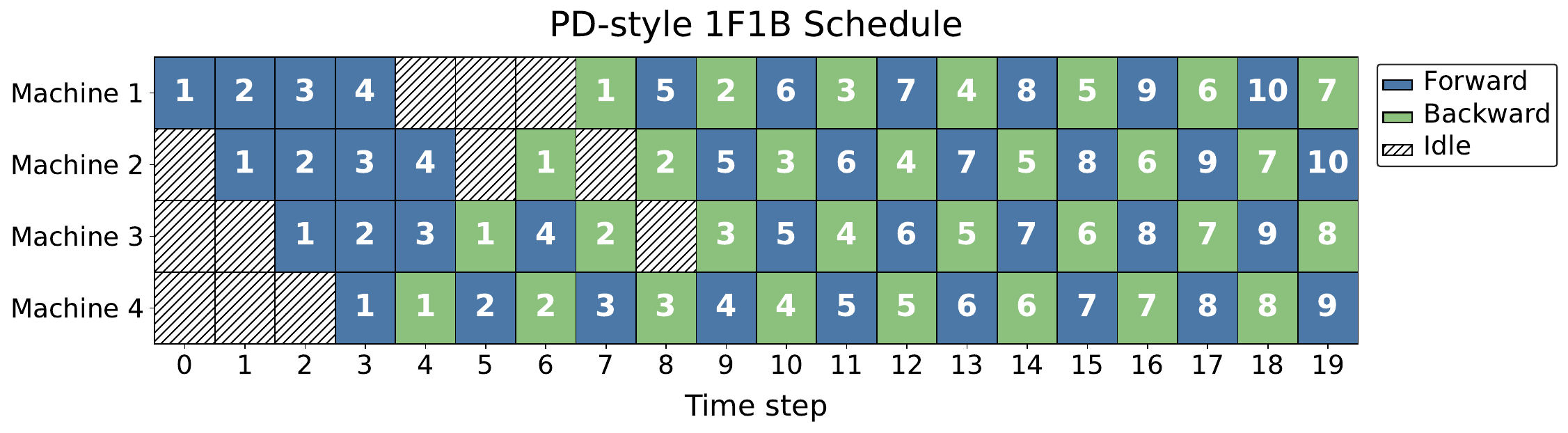}
        \caption{\algname{PD} 1F1B schedule fills most of the bubble after startup.}
        \label{fig:pipedream-1f1b}
    \end{subfigure}
    \caption{Pipeline scheduling for a four-stage model. The naive schedule is sequential and wastes hardware; \algname{PD} overlaps multiple microbatches to keep stages busy.}
    \label{fig:naive-vs-pipedream}
\end{figure}

\paragraph{From schedules to optimization updates.}
A pipeline schedule determines not only which device is active, but also which parameter version is used when a gradient is computed. We distinguish two notions of progress. The first is the \emph{block-update budget} $K$, which counts how many individual stage/block updates have been applied. This is the natural unit for the convergence theory. The second is \emph{simulator time}, which counts discrete pipeline time steps and is closer to a wall-clock proxy: in one simulator step, several stages may perform work in parallel. We use $K$ for theoretical comparisons and simulator time for experiments that study hardware-utilization tradeoffs. This distinction is important because increasing the number of stages can increase parallel work per time step while also increasing the staleness of the gradients being applied.

\section{Randomized PipeDream: A Tractable Theory}

The exact deterministic \algname{PD} schedule is difficult to analyze because the active stage, the microbatch identity, and the stage-wise weight versions are all jointly coupled by the 1F1B timeline. We therefore introduce Randomized PipeDream (\algname{RPD}), a simplified model that keeps the essential stale block-gradient structure while removing the deterministic scheduling combinatorics.

At iteration $k$, let \[w_k = \bigl(w_k^{(1)},\dots,w_k^{(S)}\bigr)\] denote the current model after $k$ block updates, where $w_k^{(s)}$ is the current value of the $s^{\text{th}}$ stage/block. \algname{RPD} first chooses a delay vector \[j_k=\bigl(j_k^{(1)},\dots,j_k^{(S)}\bigr)\] using only randomness available before the current active stage and mini-batch are sampled, with \[0 \le k-j_k^{(s)}\le \delta\]  for all $s$. It then forms a stale model \[z_k = \bigl( w_{j_k^{(1)}}^{(1)}, \dots, w_{j_k^{(S)}}^{(S)} \bigr),\] where $j_k^{(s)}$ is the past global iteration whose stage-$s$ block is used in the stale read. Conditional on this stale read, \algname{RPD} samples a pair $(s_k,m_k)$ uniformly from $\{1,\dots,S\}\times\{1,\dots,M\}$ and updates only the active block:
\begin{equation}
\label{eq:gpd-update}
w_{k+1}^{(s_k)} =
w_k^{(s_k)} - \gamma \nabla_{w^{(s_k)}} f_{m_k}(z_k),
\qquad
w_{k+1}^{(i)} = w_k^{(i)} \quad \forall i \neq s_k.
\end{equation}
Here $\delta$ is a tunable staleness bound. Unlike original \algname{PD}, which visits stages and microbatches in a highly structured order, \algname{RPD} randomizes both. This is the key simplification that makes a rigorous analysis possible. The resulting randomized version of \algname{PD} is summarized in \Cref{alg:gpd}.

\begin{algorithm}[H]
\caption{Randomized PipeDream (\algname{RPD})}
\label{alg:gpd}
\begin{algorithmic}[1]
\STATE Input: stepsize $\gamma>0$, number of stages $S$, delay bound $\delta$, initial point $w_0\in\R^d$
\STATE Initialize: $k=0$, $j_k = (j_k^{(1)}, \dots, j_k^{(S)}) = (0,\dots, 0) \in \mathbb{Z}^{S}$.
\FOR{$k=0,1,2,\dots$}
    \STATE Select a delay vector: $j_k = (j_k^{(1)}, \dots, j_k^{(S)}),$ with $0 \le k - j_k^{(i)} \le \delta \quad \forall i$
    \STATE Form a stale model: $z_k \eqdef \bigl(w_{j_k^{(1)}}^{(1)},\dots,w_{j_k^{(S)}}^{(S)}\bigr)$
    \STATE Sample $(s_k, m_k)$ uniformly from $\{1,\dots,S\} \times \{1,\dots,M\}$, cond. independently of $z_k$.
    \STATE Update the active stage block: $w_{k+1}^{(s_k)} = w_k^{(s_k)}-\gamma\,\nabla_{w^{(s_k)}} f_{m_k}(z_k)$
    \STATE Keep the remaining blocks unchanged: $w_{k+1}^{(i)}=w_k^{(i)} \qquad \forall i\neq s_k$
\ENDFOR
\end{algorithmic}
\end{algorithm}

We now analyze \Cref{alg:gpd} as a randomized stale block-SGD method. Let $\Delta_0 := f(w_0)-f_\star$. Two elementary facts drive the proof. First, because $(s_k,m_k)$ is sampled uniformly after the stale read is fixed, the randomized block update is an unbiased estimate of the full gradient direction up to the factor $1/S$. Second, trajectory-bounded block gradients and bounded delay imply that the stale read cannot be too far from the current iterate: $\norm{w_k-z_k} \le \gamma \delta G$. Combining these facts with $L$-smoothness gives the following convergence guarantee.

\begin{theorem}[Convergence of \algname{RPD}]
\label{thm:gpd}Consider \eqref{eq:objective} and the \algname{RPD} update \eqref{eq:gpd-update}. Suppose that, at each iteration, the delay vector is fixed before sampling $(s_k,m_k)$, and that conditional on the past and this delay vector, $(s_k,m_k)$ is sampled uniformly. Suppose further that 
\begin{itemize}
\item[(i)] $f$ is lower bounded, 
\item[(ii)] $L$-smooth,
\item [(iii)] there exists $G\geq 0$ such that the block gradients along the stale-read trajectory satisfy \[\norm{\nabla_{w^{(s)}} f_m(z_k)} \le G\] for all $k,m,s$ almost surely, and 
\item [(iv)] the stage-wise staleness is bounded by $\delta\geq 0$. 
\end{itemize}
If the stepsize satisfies $0 < \gamma \le 1/L$, then for every $K \ge 1$, the iterates of \algname{RPD} satisfy \begin{equation}
\label{eq:gpd-rate}
\frac{1}{K}\sum_{k=0}^{K-1}\E\bigl[\norm{\nabla f(w_k)}^2\bigr]
\le
\frac{2S\Delta_0}{\gamma K}
\;+\;
\gamma S L G^2
\;+\;
\gamma^2 L^2 \delta^2 G^2.
\end{equation}
\end{theorem}
The proof is given in Appendix~\ref{sec:gpd-proof}. The bound separates three effects. The first term is the usual optimization term and decreases with the number of block updates $K$. The second term is the stochastic block-update price: at each iteration, \algname{RPD} updates only one of the $S$ blocks. The third term is the stale-read penalty. Its dependence on $\delta^2$ is the key structural fact that will later explain the scaling behavior of \algname{PD}.

For later comparison, it is useful to tune the stepsize in \eqref{eq:gpd-rate}. If we choose \[\gamma = \min\!\left\{\frac{1}{L}, \sqrt{\frac{2\Delta_0}{L G^2 K}}\right\},\] then the \algname{RPD} rate becomes
\[
\frac{1}{K}\sum_{k=0}^{K-1}\E\bigl[\norm{\nabla f(w_k)}^2\bigr]
=
\cO\!\left(\frac{S}{\sqrt K} + \frac{\delta^2}{K}\right).
\]

\paragraph{From \algname{RPD} back to \algname{PD}.}
The point of \algname{RPD} is not to claim that the deterministic \algname{PD} scheduler samples stages and mini-batches uniformly. Rather, \algname{RPD} is a deliberately simplified stale block-SGD model: it preserves the feature of \algname{PD} that matters for the analysis---gradients are computed on stage-wise stale models---while replacing the deterministic 1F1B combinatorics by uniform sampling. 

The link back to the original 1F1B method is captured by the next result.
\begin{proposition}[Steady-state \algname{PD} delay]
\label{prop:pipedream-delay}Embed a steady-state $S$-stage \algname{PD} {\rm 1F1B} execution into the global-history model of \algname{RPD}. Under the slot-batched global-history embedding used by our simulator, the worst-case delay satisfies
\begin{equation}
\label{eq:pipedream-delay}
\delta_{\mathrm{PD}} \eqdef S^2 - \frac{S}{2} + O(1).
\end{equation}
For even $S$, this becomes the exact identity \[\delta_{\mathrm{PD}}=S^2-\frac{S}{2}.\]
\end{proposition}

The proof is given in Appendix~\ref{sec:delta-scale-proof}. This proposition identifies the delay parameter that should be used when \algname{RPD} is used as a proxy for the original deterministic \algname{PD} schedule. In particular, although \algname{RPD} randomizes the active stage and mini-batch, the worst-case stage-wise staleness produced by steady-state 1F1B \algname{PD} scales as $\delta_{\mathrm{PD}} = S^2 - \nicefrac{S}{2} + O(1)$. Under the slot-batched convention, each backward operation in a simulator slot appends one block update to the global history. Thus, to model the delay scale of real \algname{PD} within the \algname{RPD} abstraction, we should instantiate \algname{RPD} with $\delta=\delta_{\mathrm{PD}}$.

Thus \Cref{thm:gpd} should be read as a convergence theorem for a randomized \algname{PD}-style abstraction, not as a direct theorem for deterministic \algname{PD}. The role of \Cref{prop:pipedream-delay} is to transfer the scheduler-induced delay scale of \algname{PD} into this abstraction. 

\begin{figure*}[t]
    \centering
    \begin{subfigure}[t]{0.48\textwidth}
        \centering
        \includegraphics[width=\linewidth]{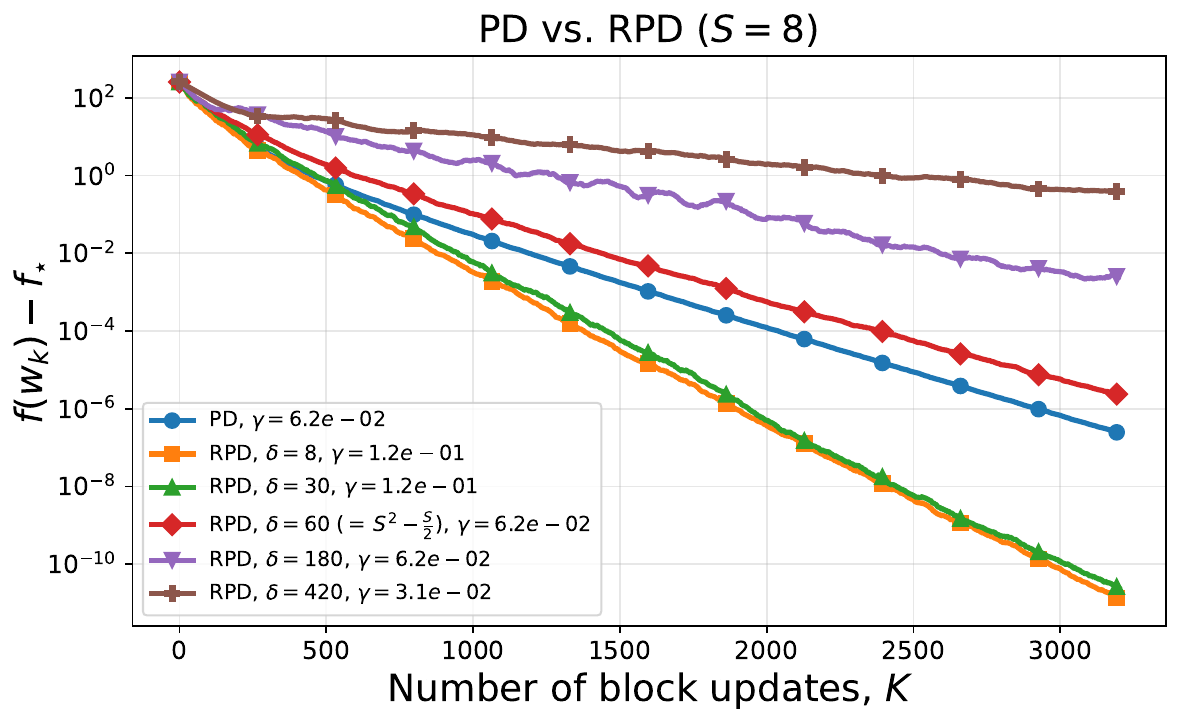}
        \caption{Random quadratic objective: best-tuned \algname{PD} and \algname{RPD} trajectories. $S=8$, batch size is $10$, $M=60$, trained for $5$ epochs.}
        \label{fig:rpd-validation-and-scaling}
    \end{subfigure}
    \hfill
    \begin{subfigure}[t]{0.48\textwidth}
        \centering
        \includegraphics[width=\linewidth]{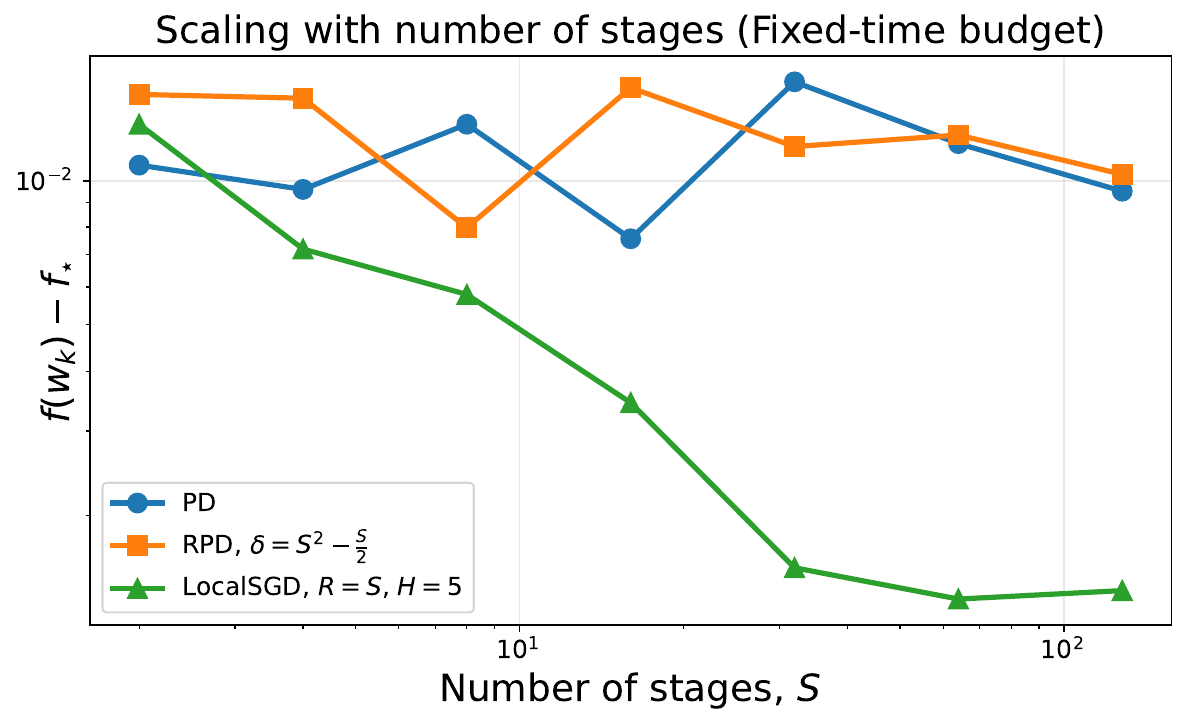}
        \caption{Logistic regression: final objective versus the number of stages $S$ under a fixed simulator-time budget. Batch size is $10$, $M=60$, trained for $5$ epochs. $H=5$, $\lambda=10^{-4}$.}
        \label{fig:fixed-time-scaling}
    \end{subfigure}
    \caption{The left plot validates \algname{RPD} as a theoretical proxy for \algname{PD}: when \algname{RPD} is instantiated in the delay regime predicted for steady-state 1F1B execution, its trajectory closely matches that of \algname{PD} on the quadratic objective. The right plot studies scaling on logistic regression: for each method and each stage count $S$, we measure the final objective reached under the same fixed simulator-time budget. $d=512$.}

    \label{fig:gpd-validation}
\end{figure*}

\section{Scaling Limits of PipeDream}

By \Cref{prop:pipedream-delay}, steady-state \algname{PD} induces $\delta = \delta_{\mathrm{PD}} = \Theta(S^2)$. Plugging this into \Cref{eq:gpd-rate} yields the staleness contribution $\Theta(\gamma^2 S^4)$, or $\Theta(S^4/K)$ in the tuned-rate form. Thus the stale-read part of the worst-case theorem has quartic dependence on the number of pipeline stages.

This does \emph{not} mean that throughput scaling disappears; \algname{PD} still reduces idle time very effectively. The point is instead that better utilization can come with a rapidly growing optimization penalty. Under a fixed computational budget, i.e., a fixed number of block updates or iterations, increasing the number of stages makes the final objective decrease more slowly. Reaching a target solution quality therefore requires substantially more optimization steps as $S$ grows.

\section{LocalSGD as an Alternative}

The scaling diagnosis above suggests that the central weakness of \algname{PD} is not pipelining itself, but the particular way in which \algname{PD} buys utilization: it keeps stages busy by allowing different parts of the model to participate in gradients computed from old stage-local versions. A natural alternative is to recover consistency through occasional synchronization. This leads to \algname{LocalSGD}: instead of maintaining one stale pipeline trajectory, we maintain several full-model trajectories, execute them through the same pipeline stages, and periodically average them.

In standard \algname{LocalSGD}, each replica stores a full parameter vector, performs several local SGD steps, and then synchronizes with the other replicas by model averaging. In our pipeline-parallel setting, each full replica is itself partitioned across the same $S$ stages used by the pipeline. We write
\[
w_{q,r}
=
\bigl(w_{q,r}^{(1)},\dots,w_{q,r}^{(S)}\bigr)\in\R^d,
\qquad
w_{q,r}^{(s)}\in\R^{d_s},
\]
for the logical model of replica $r$ after $q$ local steps. Here $s$ indexes the pipeline stage, $r$ indexes the replica, and $q$ indexes local full-model steps performed by each replica.

This model uses additional replica memory: \algname{LocalSGD} stores $R$ logical copies of the model, partitioned across the stages, whereas the basic \algname{PD} simulator follows a single stale trajectory with weight stashing. Our theory and plots compare optimization progress under matched block-update or simulated-time budgets; they do not charge this extra memory as a separate resource cost.

A logical \algname{LocalSGD} step of replica $r$ is realized by a full forward/backward pass through the $S$-stage pipeline. During this pass, each stage updates its own block $w_{q,r}^{(s)}$ when the corresponding backward computation is executed. After every $H$ local steps, the replicas are synchronized by stage-wise averaging: 
$$\bar w_q^{(s)} = \frac1R\sum_{r=1}^R w_{q,r}^{(s)}, \qquad w_{q,r}^{(s)} \leftarrow \bar w_q^{(s)} \qquad \forall r\in\{1,\dots,R\},\ s\in\{1,\dots,S\}.$$
The resulting \algname{LocalSGD} algorithm is summarized in \Cref{alg:localsgd-1f1b} located in Appendix~\ref{sec:localsgd-rate-derivation}. 

In all experiments, we choose $R=S$ in order to keep the pipeline highly utilized. The corresponding 1F1B timeline generation procedure, including the local-step dependencies and synchronization barriers, is given in Appendix~\ref{sec:localsgd-1f1b-schedule}.

\paragraph{Relation to standard \algname{LocalSGD}.}
\Cref{alg:localsgd-1f1b} is a stage-distributed version of the usual \algname{LocalSGD}. At the optimization level, each replica performs $H$ local stochastic-gradient steps and then the replicas are averaged. At the systems level, one such local step is not executed as one atomic full-vector operation: because the model is split across $S$ stages, the forward and backward passes are scheduled through the pipeline, and each stage updates its own block when its backward computation is executed.

This distinction matters for the iteration count. For \algname{RPD} and \algname{PD}, one iteration means one stage/block update. We therefore let $K$ denote the common block-update budget used in the theory comparison. For \algname{LocalSGD}, we let $T$ denote the number of logical full-model local SGD steps performed by each replica. Since one logical local step updates all $S$ blocks of one replica, and there are $R$ replicas, the common block-update budget is $K = RST$.

The parameter $H$ is the main scheduling knob. If $H$ is small, synchronization is frequent and the system incurs visible synchronization bubbles (\Cref{fig:localsgd-small-h}). If $H$ is large, synchronization becomes rare and the execution visually starts to resemble \algname{PD}: machines spend most of their time computing, with only occasional global coordination (\Cref{fig:localsgd-large-h}).

\begin{figure}[t]
    \centering
    \begin{subfigure}[t]{0.90\linewidth}
        \centering
        \includegraphics[width=\linewidth]{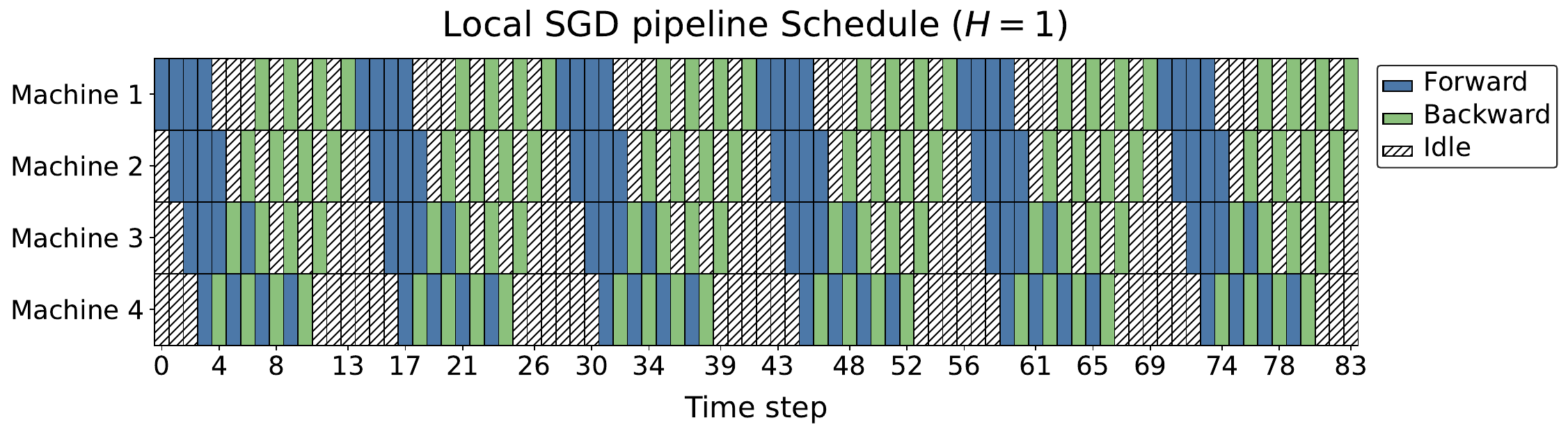}
        \caption{$H=1$, $R=S=4$: frequent synchronization introduces visible bubbles.}
        \label{fig:localsgd-small-h}
    \end{subfigure}
    \hfill
    \begin{subfigure}[t]{0.90\linewidth}
        \centering
        \includegraphics[width=\linewidth]{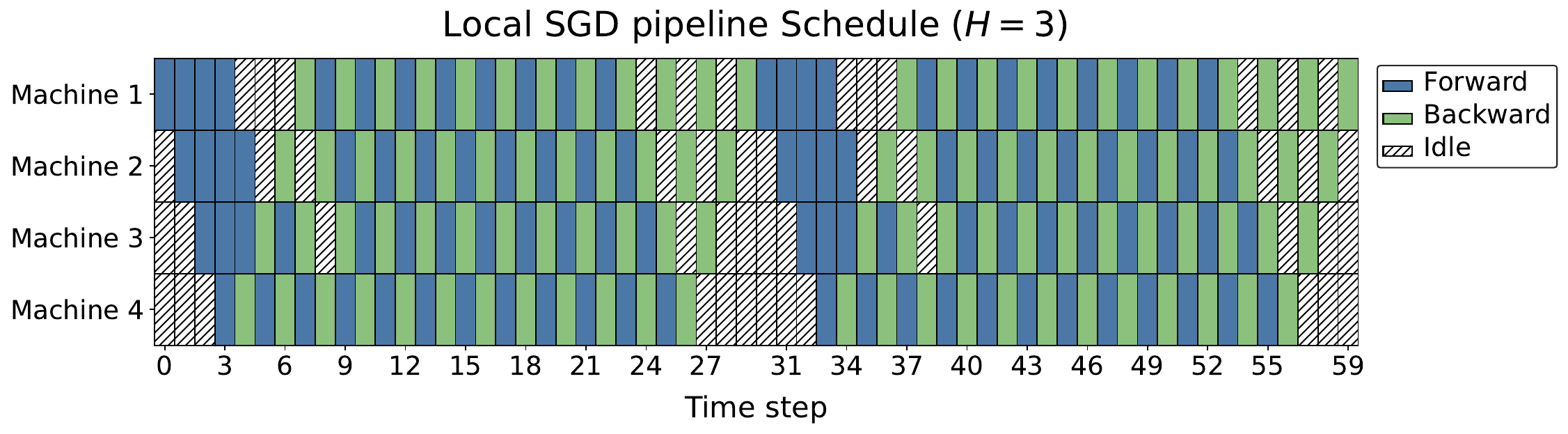}
        \caption{$H=3$, $R=S=4$: synchronization is sparse and the schedule approaches pipeline-style utilization.}
        \label{fig:localsgd-large-h}
    \end{subfigure}
    \caption{Scheduling intuition for \algname{LocalSGD}. Increasing the number of local steps $H$ shrinks synchronization overhead but increases replica drift between averaging events.}
    \label{fig:localsgd-schedule}
\end{figure}

This trade-off is qualitatively different from the trade-off in \algname{PD}. \algname{PD} reduces hardware idle time by accepting stale gradients tied to old weight versions. \algname{LocalSGD} avoids that particular staleness mechanism, but it pays for it with synchronization pauses and with drift between replicas.

\begin{figure*}[t]
    \centering
    \begin{subfigure}[t]{0.48\textwidth}
        \centering
        \includegraphics[width=\linewidth]{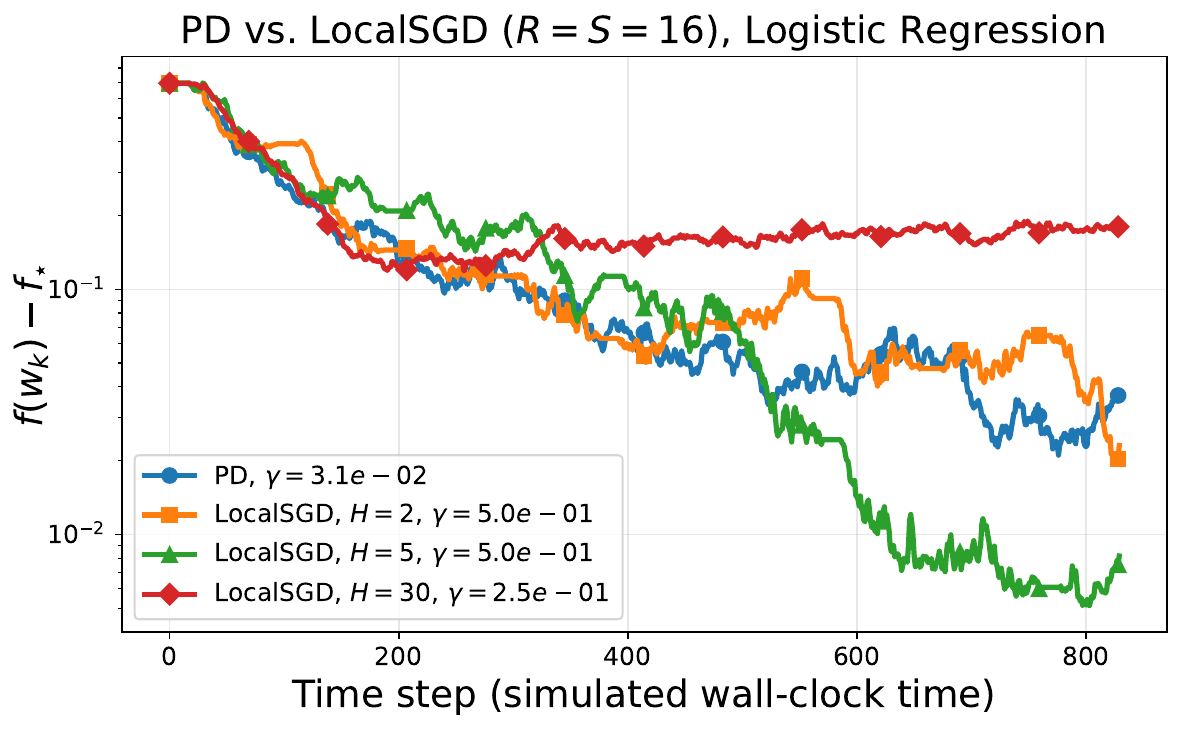}
        \caption{Logistic regression.}
        \label{fig:pd-vs-localsgd-logreg}
    \end{subfigure}
    \hfill
    \begin{subfigure}[t]{0.48\textwidth}
        \centering
        \includegraphics[width=\linewidth]{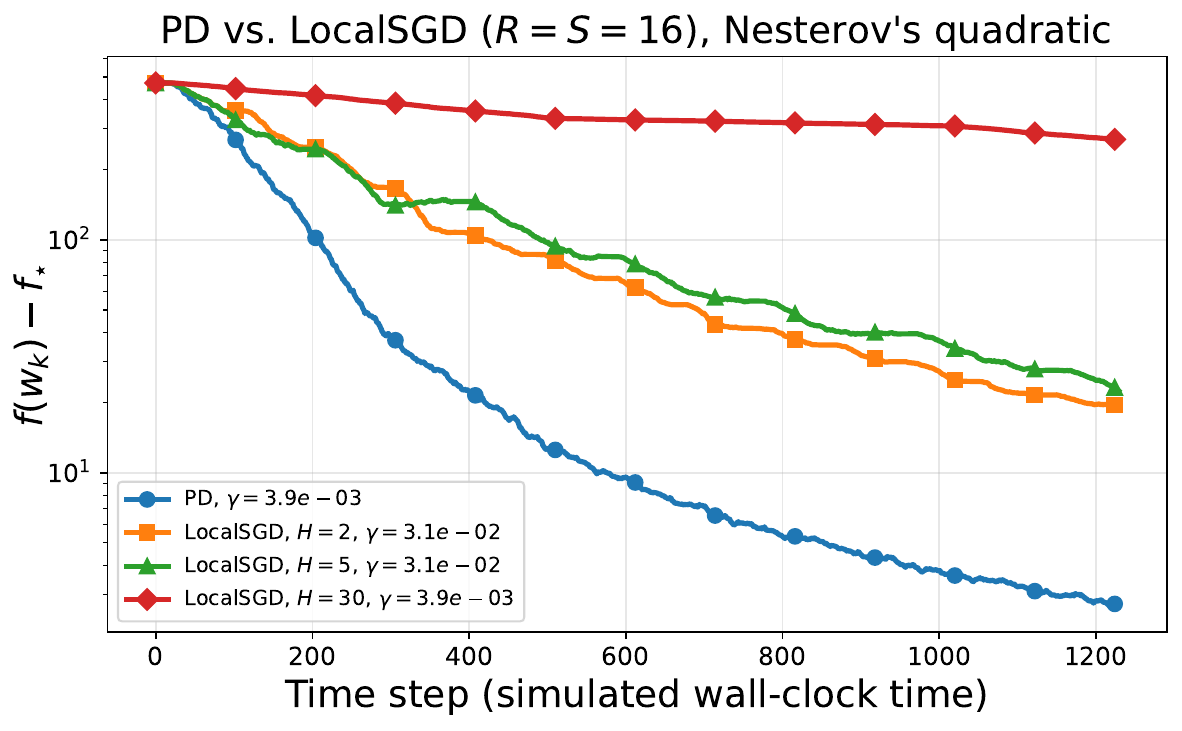}
        \caption{Nesterov-style tridiagonal quadratic objective.}
        \label{fig:pd-vs-localsgd-quad}
    \end{subfigure}
    \hfill
    \begin{subfigure}[t]{0.48\textwidth}
        \centering
        \includegraphics[width=\linewidth]{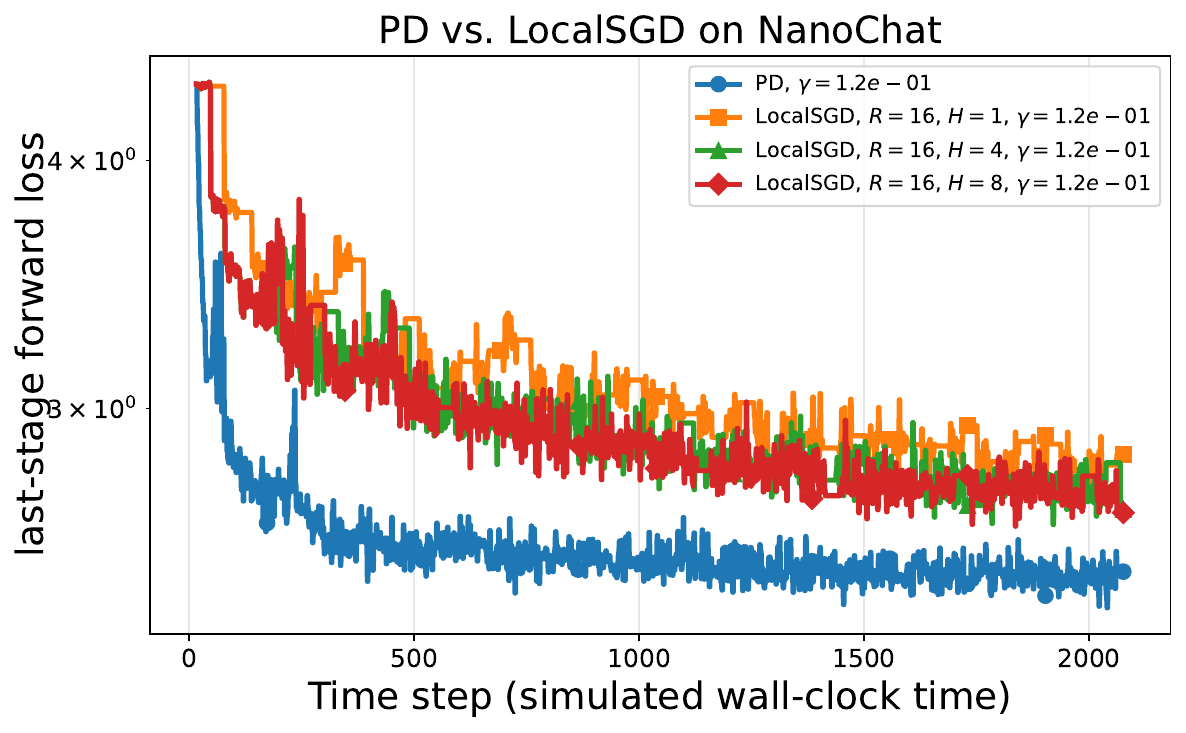}
        \caption{NanoChat character language modeling.}
        \label{fig:pd-vs-localsgd-nanochat}
    \end{subfigure}
    \caption{\algname{PD} versus \algname{LocalSGD} at $S=16$. Panels (a,b) use synthetic finite-sum objectives with $d=512$, batch size $10$, and $M=60$; panel (c) uses an $11.1$M-parameter NanoChat-style Transformer on Tiny Shakespeare with sequence length $128$, batch size $32$, embedding dimension $256$, $8$ attention heads, and tuned learning rates. \algname{PD} is stronger on the quadratic and NanoChat tasks, while \algname{LocalSGD} is stronger on logistic regression.}
    \label{fig:pd-vs-localsgd}
\end{figure*}

\paragraph{Convergence analysis.}
We use the standard nonconvex \algname{LocalSGD} analysis of \citet{yu2019parallel} in the homogeneous-worker case. To avoid overloading the finite-sum components $f_m$ from \eqref{eq:objective}, denote the objective assigned to replica $r$ in the standard LocalSGD formulation by $\phi_r$. In our setting, $\phi_1=\cdots=\phi_R=f,$ where $R$ is the number of replicas, i.e., the number of parallel local trajectories. Thus, the worker count in \citet{yu2019parallel} corresponds to our $R$.

Let $\bar w_q := \frac1R\sum_{r=1}^R w_{q,r}$ be the averaged logical model after $q$ local steps per replica. Assume that $f$ is lower bounded by $f_\star$ and $L$-smooth. For uniformly sampled mini-batches $m_{q,r}\in\{1,\dots,M\}$, assume
\[
\E\!\left[
\nabla f_{m_{q,r}}(w_{q,r})
\mid w_{q,r}
\right]
=
\nabla f(w_{q,r}),
\]
and
\[
\E\!\left[
\left\|
\nabla f_{m_{q,r}}(w_{q,r})
-
\nabla f(w_{q,r})
\right\|^2
\right]
\le \sigma^2,
\qquad
\E\!\left[
\left\|
\nabla f_{m_{q,r}}(w_{q,r})
\right\|^2
\right]
\le B^2.
\]

In all experiments we choose $R=S$ to keep the 1F1B pipeline highly utilized. Assume $K$ is divisible by $RS$; otherwise one can replace $K/(RS)$ by $\lfloor K/(RS)\rfloor$ throughout. Since the common block-update budget satisfies $K=RST$, the choice $R=S$ gives $T=K/S^2$. Applying the standard \algname{LocalSGD} guarantee under this budget gives, for $0<\gamma\le 1/L$,
\begin{equation}
\label{eq:localsgd-rate-K}
\frac{S^2}{K}\sum_{q=0}^{K/S^2-1}
\E\bigl[\|\nabla f(\bar w_q)\|^2\bigr]
\le
\frac{2S^2\Delta_0}{\gamma K}
+
4\gamma^2H^2B^2L^2
+
\frac{L\gamma\sigma^2}{S}.
\end{equation}
The derivation from the standard rate of \citet{yu2019parallel} and the substitution $T=K/(RS)$ are given in Appendix~\ref{sec:localsgd-rate-derivation}.

The first term of \eqref{eq:localsgd-rate-K} is the optimization term, the second is the local-drift penalty caused by taking $H$ unsynchronized local steps, and the third is the stochastic-gradient variance term reduced by averaging across the $S$ replicas.

The tuned form below uses the corollary of \citet{yu2019parallel}: choose $\gamma=\min\{1/L,\sqrt{R}/(L\sqrt{T})\}$. In the regime where the second term is active, $R=S$ and $T=K/S^2$ give the
$R=S$ rate under the common block-update budget:
\begin{equation}
\label{eq:localsgd-tuned-rate-K}
\frac{S^2}{K}\sum_{q=0}^{K/S^2-1}
\E\bigl[\|\nabla f(\bar w_q)\|^2\bigr]
=
\cO\!\left(
\sqrt{\frac{S}{K}}
+
\frac{S^3H^2}{K}
\right).
\end{equation}
The \algname{LocalSGD} guarantee is for the logical averaged model $\bar w_q$, whereas \algname{RPD}/\algname{PD} are measured along a single stale block-update trajectory. In contrast to the \algname{RPD}/\algname{PD} abstraction, the \algname{LocalSGD} bound does not contain a stale-weight term of the form $\delta^2/K$. The dominant additional error is instead the local-replica drift caused by taking $H$ unsynchronized steps.

\begin{table*}[t]
\centering
\footnotesize
\setlength{\tabcolsep}{4pt}
\renewcommand{\arraystretch}{1.45}

\newcolumntype{L}[1]{>{\raggedright\arraybackslash}m{#1}}
\newcolumntype{C}[1]{>{\centering\arraybackslash}m{#1}}

\begin{tabular}{@{}L{0.11\textwidth}C{0.32\textwidth}C{0.19\textwidth}L{0.26\textwidth}@{}}
\toprule
\textbf{Method}
&
\textbf{Bound for block-update budget $K$}
&
\textbf{Tuned rate}
&
\textbf{Assumptions}
\\
\midrule

\algname{RPD}
&
$\begin{aligned}
&\frac{2S\Delta_0}{\gamma K}
+\gamma S L G^2
+\gamma^2 L^2\delta^2G^2
\end{aligned}$
&
$\begin{aligned}
\cO\!\left(
\frac{S}{\sqrt K}
+
\frac{\delta^2}{K}
\right)
\end{aligned}$
&
$L$-smooth, lower bounded, trajectory-bounded block gradients, uniform stage-batch sampling after delay selection, delay $\le \delta$, $0<\gamma\le 1/L$
\\

\addlinespace[0.25em]

\algname{RPD} proxy for \algname{PD}
&
$\begin{aligned}
&\frac{2S\Delta_0}{\gamma K}
+\gamma S L G^2
+\gamma^2 L^2\delta_{\mathrm{PD}}^2G^2
\end{aligned}$
&
$\begin{aligned}
\cO\!\left(
\frac{S}{\sqrt K}
+
\frac{S^4}{K}
\right)
\end{aligned}$
&
same as \algname{RPD}, with $\delta = \delta_{\mathrm{PD}}$
\\

\addlinespace[0.25em]

\algname{LocalSGD}
&
$\begin{aligned}
&\frac{2S^2\Delta_0}{\gamma K}
+4\gamma^2H^2B^2L^2
+\frac{L\gamma\sigma^2}{S}
\end{aligned}$
&
$\begin{aligned}
\cO\!\left(
\sqrt{\frac{S}{K}}
+
\frac{S^3H^2}{K}
\right)
\end{aligned}$
&
$L$-smooth, lower bounded, unbiased gradients, bounded variance and second moments, $0<\gamma\le 1/L$
\\

\bottomrule
\end{tabular}

\caption{Comparison of rates for the studied optimization models. The \algname{PD} row is obtained by inserting the steady-state \algname{PD} delay scale $\delta = \delta_{\mathrm{PD}}$ into the \algname{RPD} theorem. For \algname{RPD} and \algname{PD}, $K$ directly counts stage/block updates. For \algname{LocalSGD}, $T$ denotes the number of logical full-model local SGD steps per replica, and the staged implementation satisfies $K=RST$.}
\label{tab:theory-comparison}
\end{table*}

\paragraph{Which method is better in which regime?}
The two bounds should not be read as proving that one schedule is uniformly
better than the other. Instead, they expose two different failure modes.
\algname{PD} can suffer when stage-wise stale reads dominate, especially as
\(\delta_{\mathrm{PD}}\) grows with \(S\). \algname{LocalSGD} removes this
particular stale-version mechanism, but can suffer from local-replica drift
when \(H\) is large. The empirical question is therefore not whether
\algname{LocalSGD} always improves over \algname{PD}, but which error source is
more damaging for a given objective and scaling regime.

\Cref{tab:theory-comparison} highlights the basic tradeoff between the two methods under the common block-update budget $K$. For \algname{PD}, the main difficulty is pipeline staleness: after substituting $\delta=\delta_{\mathrm{PD}}=\Theta(S^2)$ into the \algname{RPD} bound, the staleness contribution becomes $\Theta(S^4/K)$. Thus, as the number of stages $S$ increases, \algname{PD} becomes increasingly sensitive to stale gradients. \algname{LocalSGD} avoids this particular stale-version mechanism, but pays a different price: the \algname{LocalSGD} rate contains a local-drift term of order $S^3H^2/K$. In this sense, the theory isolates a staleness-versus-drift tradeoff: \algname{PD} is harmed by deep pipelines through the $\Theta(S^4/K)$ staleness term, while \algname{LocalSGD} is harmed by long unsynchronized local trajectories through the $H^2$ drift term.

\Cref{tab:theory-comparison} also clarifies the role of $H$. Taking $H$ very large makes \algname{LocalSGD} visually resemble pipeline execution and reduces synchronization bubbles, but the \algname{LocalSGD} bound shows that the drift penalty grows quadratically in $H$. Taking $H$ very small keeps the method close to synchronized mini-batch SGD, but loses more simulator time to synchronization barriers (\Cref{fig:localsgd-schedule}). Hence the design problem for \algname{LocalSGD} is to choose $H$ large enough to reduce idle time, but not so large that replica drift dominates. We therefore tune the synchronization period $H$ empirically.

\section{Experiments}
\label{sec:experiments}

Our experiments are designed to test three questions. First, does the randomized
\algname{RPD} abstraction predict the behavior of deterministic 1F1B
\algname{PD} when instantiated with the delay scale \(\delta_{\mathrm{PD}}\)?
Second, does the degradation predicted by the \(\delta^2\) term become visible
as the number of stages \(S\) increases? Third, is the resulting
staleness-versus-synchronization tradeoff objective-dependent?

We answer these questions on three regimes: controlled quadratic objectives,
nonlinear logistic regression, and an $11.1$M-parameter NanoChat-style
Transformer language-modeling task \citep{karpathy2025nanochat, karpathy2022nanogpt}. The first two allow exact full-objective
measurement, while the Transformer experiment tests whether the observed
optimization behavior persists beyond synthetic finite-sum objectives.

\paragraph{Experimental objectives.}
We compare \algname{PD}, \algname{RPD}, and \algname{LocalSGD} on synthetic quadratic and logistic-regression finite-sum objectives whose full losses can be evaluated exactly. Each method is tuned separately, and \algname{PD}/\algname{LocalSGD} are run by replaying static 1F1B timelines in a discrete-event simulator. Objective definitions, seeds, learning-rate grids, simulator budgets, simulator scope, and compute resources are given in Appendix~\ref{sec:experiments-extra}.

\paragraph{\algname{PD} versus \algname{RPD}.}
The left panel of \Cref{fig:gpd-validation} shows that \algname{RPD} closely tracks deterministic \algname{PD} on the quadratic objective when instantiated with the delay scale predicted by \Cref{prop:pipedream-delay}, namely \(\delta=\delta_{\mathrm{PD}}\). This supports the use of \algname{RPD} as a tractable proxy for the dominant optimization effect of 1F1B execution: stage-wise stale model versions.

\paragraph{Scaling with the number of stages.}
The right panel of \Cref{fig:gpd-validation} studies logistic regression under a fixed simulator-time budget while varying $S$. This is an empirical wall-clock-proxy experiment rather than a direct corollary of \Cref{thm:gpd}, whose comparison unit is the block-update budget $K$. For \algname{PD} and \algname{LocalSGD}, one simulator tick is one row of the static pipeline timeline; for \algname{RPD}, the same clock determines the number of block updates by counting how many backward operations \algname{PD} completes within the matched timeline at that value of $S$. Thus increasing $S$ can give \algname{RPD}/\algname{PD} more block updates within the same simulated-time horizon, while also increasing the delay scale. Even with this parallel-work advantage, stale-gradient pipeline methods deteriorate as $S$ grows, consistent with the $\delta_{\mathrm{PD}}=\Theta(S^2)$ delay law and the $\delta^2$ term in the convergence bound. \algname{LocalSGD} scales more favorably in this experiment.

\paragraph{\algname{PD} versus \algname{LocalSGD}.}
The experiments isolate optimization effects rather than full systems performance. We use synthetic quadratic and logistic-regression objectives with exactly measurable losses, and a NanoChat \citep{karpathy2025nanochat, karpathy2022nanogpt} character-level language-modeling task on Tiny Shakespeare \citep{karpathy2015charrnn} whose $11{,}123{,}200$-parameter causal Transformer is split across $S=16$ stages. All methods are tuned separately and replay static 1F1B timelines; details are in Appendix~\ref{sec:experiments-extra}. \Cref{fig:gpd-validation} shows that \algname{RPD} tracks deterministic \algname{PD} when instantiated with $\delta=\delta_{\mathrm{PD}}$, and that stale-gradient methods deteriorate as $S$ grows under a fixed simulator-time budget. \Cref{fig:pd-vs-localsgd} shows that \algname{PD} is better on the quadratic and NanoChat tasks, while \algname{LocalSGD} is better on logistic regression.

\paragraph{Results summary.}
Overall, the experiments support the theoretical picture while showing that the practical tradeoff is objective-dependent. On the random quadratic, \algname{RPD} with the \algname{PD}-predicted delay closely follows deterministic \algname{PD}, suggesting that the stale-block abstraction captures the dominant optimization effect of weight stashing. On logistic regression, increasing the number of stages under a fixed simulator-time budget makes stale-gradient pipeline methods less effective, while \algname{LocalSGD} is more robust. At $S=16$, however, \algname{PD} still obtains the best tuned performance on the tridiagonal quadratic and NanoChat experiments, so synchronization is not uniformly preferable to stale pipeline execution.

\section{Conclusion}

We studied \algname{PD} as an optimization algorithm induced by a pipeline
schedule. Directly analyzing deterministic 1F1B \algname{PD} is difficult
because the active stages, microbatches, and stage-wise weight versions are
coupled through the schedule and weight stashing. To make this structure
tractable, we introduced \algname{RPD}, a randomized stale block-SGD proxy that
preserves the key optimization feature of \algname{PD}: gradients are evaluated
on models assembled from stale stage-local parameter versions. For this proxy,
we proved a nonconvex convergence guarantee and connected it back to real 1F1B
\algname{PD} by showing that the induced worst-case delay scales as
\(\delta_{\mathrm{PD}} = S^2 - \nicefrac{S}{2} + O(1)\). Substituting this
delay into the \algname{RPD} theorem exposes a \(\Theta(\gamma^2S^4)\)
stale-read term, or \(\Theta(S^4/K)\) after tuned stepsize substitution,
suggesting that deeper pipelines can incur a rapidly growing
optimization cost.

We also compared this stale-version mechanism with \algname{LocalSGD}, which
avoids PipeDream-style cross-stage staleness but introduces synchronization
bubbles and local-replica drift. Our experiments support this tradeoff view:
\algname{PD} is stronger on the quadratic objective and on the NanoChat-style
language-modeling task, while \algname{LocalSGD} is stronger on logistic
regression at larger numbers of stages. A natural next step is to analyze
deterministic \algname{PD} directly, incorporating the exact 1F1B scheduling
pattern and its deterministic staleness indices into the convergence argument.

\begin{ack}
    The research reported in this publication was supported by funding from King Abdullah University of Science and Technology (KAUST): i) KAUST Baseline Research Scheme, ii) CRG Grant ORFS-CRG12-2024-6460, and iii) Center of Excellence for Generative AI, under award number 5940.
\end{ack}

\bibliographystyle{plainnat}
\bibliography{references}


\clearpage
\appendix

\section{Table of Frequently Used Notation} \label{sec:notation}

\begin{table}[H]
\caption{Notation table}
\label{tab:notation_table}
\centering
\small
\setlength{\tabcolsep}{3pt}
\begin{tabularx}{\linewidth}{@{}r c X@{}}
\toprule
    $d$ & -- & Total dimension of the model parameters. \\
    $M$ & -- & Number of finite-sum components, corresponding to data mini-batches in our experiments. \\
    $S$ & -- & Number of pipeline stages (or partitioned blocks of the model). \\
    $d_s$ & -- & Dimension of the $s^{\text{th}}$ block of the model, such that $\sum_{s=1}^S d_s = d$. \\
    $w$ & -- & Full model parameter vector, $w \in \R^d$. \\
    $w^{(s)}$ & -- & The $s^{\text{th}}$ block of the model weights, corresponding to the $s^{\text{th}}$ stage, $w^{(s)} \in \R^{d_s}$. \\
    $f(w)$ & -- & Global objective function to minimize. \\
    $f_m(w)$ & -- & Loss function evaluated on the $m^{\text{th}}$ data batch. \\
    $f_\star$ & -- & Global lower bound of the objective function $f$. \\
    $\Delta_0$ & $\eqdef$ & Initial objective gap; for \algname{RPD}/\algname{PD}, $\Delta_0=f(w_0)-f_\star$, and for \algname{LocalSGD}, $\Delta_0=f(\bar w_0)-f_\star$. \\
    $k$ & -- & Global iteration counter (for block updates). \\
    $K$ & -- & Block-update budget. For \algname{RPD} and \algname{PD}, $K$ counts stage/block updates directly. For \algname{LocalSGD}, $K=RST$, where $T$ is the number of logical local steps per replica. \\
    $s_k, m_k$ & -- & Active stage index and batch index sampled uniformly at iteration $k$, after the delay vector is selected. \\
    $j_k^{(s)}$ & -- & The global history index whose stage-$s$ block is used in the stale read at iteration $k$. \\
    $z_k$ & $\eqdef$ & $\bigl(w_{j_k^{(1)}}^{(1)},\dots,w_{j_k^{(S)}}^{(S)}\bigr)$ -- Stale model read at iteration $k$. \\
    $\widehat{z}_k$ & -- & Extrapolated stale model used in the \algname{NAG-RPD}. \\
    $\delta$ & -- & Upper bound on the stage-wise staleness/delay, such that $k - j_k^{(s)} \le \delta$. \\
    $\delta_{\mathrm{PD}}$ & -- & Exact worst-case global-history delay of the steady-state \algname{PD} 1F1B schedule. \\
    $\delta_k^{(s)}$ & -- & Exact global-history delay of stale block $s$ at backward-update event $k$ for a \algname{PD} 1F1B schedule. \\
    $a_k$ & -- & Active backward stage at the $k^{\text{th}}$ \algname{PD} backward-update event. \\
    $\bar \delta$ & -- & Average global-history delay over the steady-state execution for a \algname{PD} 1F1B schedule. \\
    $\gamma$ & -- & Stepsize (learning rate). \\
    $L$ & -- & Smoothness constant of the objective function $f$. \\
    $G$ & -- & Trajectory bound on the norm of the block gradients used in the analysis. \\
    $\mu$ & -- & Strong convexity or Polyak--\L ojasiewicz (PL) condition parameter. \\
    $\alpha$ & -- & Prediction/momentum parameter used in \algname{NAG-RPD}. \\
    $U_s$ & -- & Canonical embedding matrix from the block space $\R^{d_s}$ into the full space $\R^d$. \\
    $g_k$ & $\eqdef$ & $U_{s_k}\nabla_{w^{(s_k)}} f_{m_k}(z_k)$ -- Full-space update vector at iteration $k$. \\
    $R$ & -- & Number of independent replicas in \algname{LocalSGD}; in the experiments we use $R=S$. \\
    $T$ & -- & Number of logical local SGD steps performed by each replica. Block-update budget $K=RST$. \\
    $H$ & -- & Synchronization period in \algname{LocalSGD}; replicas are averaged every $H$ local steps. \\
    $q$ & -- & Logical local-step index for \algname{LocalSGD}. \\
    $r$ & -- & Replica index, $r \in \{1,\dots,R\}$. \\
    $m_{q,r}$ & -- & Mini-batch/component index sampled by replica $r$ at local step $q$. \\
    $w_{q,r}$ & -- & Full model of replica $r$ after $q$ local steps. \\
    $w_{q,r}^{(s)}$ & -- & Stage-$s$ block of replica $r$ after $q$ local steps. \\
    $\bar w_q$ & $\eqdef$ & $\frac{1}{R}\sum_{r=1}^R w_{q,r}$ -- Averaged logical model after $q$ local steps per replica. \\
    $B$ & -- & Uniform second-moment bound for stochastic gradients in the \algname{LocalSGD} analysis. \\
    $\sigma^2$ & -- & Uniform variance bound for stochastic gradients in the \algname{LocalSGD} analysis. \\
\bottomrule
\end{tabularx}
\end{table}

\clearpage

\section{Convergence of RPD}
\label{sec:gpd-proof}

We will analyze \algname{RPD} (\Cref{alg:gpd}) under a random pair sampling assumption. The original \algname{PD} schedule is deterministic and couples stage order, mini-batch order, and delay vectors through the 1F1B timeline. In this appendix we analyze the randomized abstraction used in the main text, where the delay vector is selected first and then the active stage and mini-batch are sampled uniformly. This abstraction is designed to isolate the effect of bounded stale block gradients while avoiding the combinatorial complexity of the exact scheduler.

We assume that the delay vector used at iteration $k$ is chosen before the
current active stage and mini-batch. We then sample the pair
\[
(s_k,m_k)\in \{1,\dots,S\}\times\{1,\dots,M\}
\]
randomly.

\begin{assumption}[Uniform random pair sampling after delay selection]
\label{ass:gpd_random_pair}
Let $\mathcal F_k$ denote the sigma-field generated by all randomness before the
current stage/mini-batch draw, including the current delay vector $j_k$ and
therefore the stale point $z_k$. For every $k\ge 0$, conditioned on
$\mathcal F_k$, the pair $(s_k,m_k)$ is sampled uniformly from
\[
\{1,\dots,S\}\times\{1,\dots,M\},
\]
and is independent of all previous randomness not already contained in
$\mathcal F_k$.
\end{assumption}

\begin{assumption}[Bounded staleness]
\label{ass:gpd_staleness_random}
There exists an integer $\delta\ge 0$ such that for every $k\ge 0$,
\[
z_k=\bigl(w_{j_k^{(1)}}^{(1)},\dots,w_{j_k^{(S)}}^{(S)}\bigr),
\qquad
j_k^{(s)}\le k,
\qquad
k-j_k^{(s)}\le \delta
\quad \forall s\in\{1,\dots,S\}.
\]
\end{assumption}

\subsection{Assumptions on the problem}

\begin{assumption}[Lower boundedness]
\label{ass:gpd_lower_random}
The function $f$ is bounded from below:
\[
f(w)\ge f_\star
\qquad \forall w\in\R^d.
\]
\end{assumption}

\begin{assumption}[$L$-smoothness]
\label{ass:gpd_Lsmooth_random}
The function $f:\R^d\to\R$ is differentiable and $L$-smooth:
\[
\|\nabla f(x)-\nabla f(y)\|\le L\|x-y\|
\qquad \forall x,y\in\R^d.
\]
\end{assumption}

\begin{assumption}[Trajectory-bounded block gradients]
\label{ass:gpd_bounded_block_random}
There exists $G\ge 0$ such that
\[
\|\nabla_{w^{(s)}} f_m(z_k)\|\le G
\qquad \forall k\ge 0,\ \forall s\in\{1,\dots,S\},\ \forall m\in\{1,\dots,M\},
\]
almost surely along the trajectory generated by the algorithm.
\end{assumption}

\subsection{Analysis}
Recall that the \algname{RPD} update is
\[
w_{k+1}^{(s_k)}
=
w_k^{(s_k)}-\gamma\,\nabla_{w^{(s_k)}} f_{m_k}(z_k),
\qquad
w_{k+1}^{(r)}=w_k^{(r)} \ \ \forall r\neq s_k.
\]

For convenience, define the full-space update vector
\[
g_k
\eqdef
U_{s_k}\nabla_{w^{(s_k)}} f_{m_k}(z_k),
\]
where $U_s:\R^{d_s}\to\R^d$ is the canonical embedding of block $s$ into the full space.
Then the update can be written compactly as
\[
w_{k+1}=w_k-\gamma g_k.
\]

\Needspace{8\baselineskip}
\begin{lemma}[Unbiasedness of the \algname{RPD} update]
\label{lem:gpd_unbiased_direction}Let \Cref{ass:gpd_random_pair} (uniform random pair sampling) hold. Then
\[
\Exp{S g_k\mid \mathcal F_k}
=
 \nabla f(z_k),
\]
where $\mathcal F_k$ is the sigma-field in \Cref{ass:gpd_random_pair}.
\end{lemma}

\begin{proof}
Conditioned on $\mathcal F_k$, the stale point $z_k$ is fixed. Hence
\begin{align*}
\Exp{g_k\mid \mathcal F_k}
&=
\frac{1}{SM}
\sum_{s=1}^S\sum_{m=1}^M
U_s \nabla_{w^{(s)}} f_m(z_k) \\
&=
\frac{1}{S}
\sum_{s=1}^S
U_s \left( \frac1M\sum_{m=1}^M \nabla_{w^{(s)}} f_m(z_k) \right) \\
&=
\frac1S
\sum_{s=1}^S
U_s \nabla_{w^{(s)}} f(z_k)
=
\frac1S \nabla f(z_k).
\end{align*}
\end{proof}

\begin{lemma}[Distance between current and stale models]
\label{lem:rpd-stale-distance}Let \Cref{ass:gpd_staleness_random} ($\delta$-bounded staleness) and \Cref{ass:gpd_bounded_block_random} (trajectory-bounded block gradients) hold. Then for every $k\ge 0$,
\[
\|w_k-z_k\|\le \gamma\delta G.
\]
\end{lemma}

\begin{proof}
At each iteration, only one block is updated, so
\[
\|w_{t+1}-w_t\|
=
\gamma \|g_t\|
=
\gamma \|\nabla_{w^{(s_t)}} f_{m_t}(z_t)\|
\le \gamma G.
\]
Now $z_k$ is assembled from block values that are at most $\delta$ iterations old. Hence $w_k-z_k$ consists only of block increments accumulated during the last $\delta$ iterations. Therefore
\[
\|w_k-z_k\|
\le
\sum_{t=k-\delta}^{k-1}\|w_{t+1}-w_t\|
\le
\sum_{t=k-\delta}^{k-1}\gamma G
=
\gamma\delta G.
\]
\end{proof}

We can now proceed to the main theorem.

\Needspace{16\baselineskip}
\begin{theorem}[Convergence of \algname{RPD}]
\label{thm:rpd-random-nonconvex-app}Let \Cref{ass:gpd_random_pair} (uniform random pair sampling after delay selection), \Cref{ass:gpd_lower_random} (lower boundedness), \Cref{ass:gpd_Lsmooth_random} ($L$-smoothness), \Cref{ass:gpd_staleness_random} ($\delta$-bounded staleness), \Cref{ass:gpd_bounded_block_random} (trajectory-bounded block gradients) hold.
Assume
\[
0<\gamma\le \frac1L.
\]
Then for every integer $K\ge 1$,
\begin{align}
\frac1K\sum_{k=0}^{K-1}\Exp{\|\nabla f(w_k)\|^2}
&\le
\frac{2S\bigl(f(w_0)-f_\star\bigr)}{\gamma K}
+
\gamma SL G^2
+
\gamma^2 L^2\delta^2 G^2.
\label{eq:gpd_random_nonconvex_rate}
\end{align}
\end{theorem}

\begin{proof}
By $L$-smoothness of $f$,
\[
f(w_{k+1})
\le
f(w_k)
+
\inner{\nabla f(w_k)}{w_{k+1}-w_k}
+
\frac{L}{2}\|w_{k+1}-w_k\|^2.
\]
Since $w_{k+1}=w_k-\gamma g_k$, this becomes
\[
f(w_{k+1})
\le
f(w_k)
-
\gamma\inner{\nabla f(w_k)}{g_k}
+
\frac{\gamma^2 L}{2}\|g_k\|^2.
\]
Taking conditional expectation with respect to $\mathcal F_k$, and using \Cref{lem:gpd_unbiased_direction},
\begin{equation} \label{eq:09u9f8yp9fod-fg}
\Exp{f(w_{k+1})\mid\mathcal F_k}
\le
f(w_k)
-
\frac{\gamma}{S}\inner{\nabla f(w_k)}{\nabla f(z_k)}
+
\frac{\gamma^2 L}{2}\Exp{\|g_k\|^2\mid\mathcal F_k}.
\end{equation}
In view of \Cref{ass:gpd_bounded_block_random} (trajectory-bounded block gradients), we have
$
\|g_k\|
=
\|\nabla_{w^{(s_k)}} f_{m_k}(z_k)\|
\le G,
$
and hence
\[
\Exp{\|g_k\|^2\mid\mathcal F_k} \le G^2.
\]
Plugging this estimate into \eqref{eq:09u9f8yp9fod-fg}, we get
\begin{equation}\label{eq:nihp-98fydp98ufd}
\Exp{ f(w_{k+1})\mid\mathcal F_k}
\le
f(w_k)
-
\frac{\gamma}{S}\inner{\nabla f(w_k)}{\nabla f(z_k)}
+
\frac{\gamma^2 LG^2}{2}.
\end{equation}

Let us now use the elementary inequality
\[
\inner{a}{b}
=
\frac12\|a\|^2+\frac12\|b\|^2-\frac12\|a-b\|^2
\ge
\frac12\|a\|^2-\frac12\|a-b\|^2,
\]
 with
$
a=\nabla f(w_k)$ and $
b=\nabla f(z_k)
$ to estimate the inner product in \eqref{eq:nihp-98fydp98ufd}.
This gives
\[
-\frac{\gamma}{S}\inner{\nabla f(w_k)}{\nabla f(z_k)}
\le
-\frac{\gamma}{2S}\|\nabla f(w_k)\|^2
+
\frac{\gamma}{2S}\|\nabla f(w_k)-\nabla f(z_k)\|^2.
\]
Plugging this estimate into \eqref{eq:nihp-98fydp98ufd}, we get
\begin{equation}\label{eq:-9=fp9du0f-d9gf}
\Exp{f(w_{k+1})\mid\mathcal F_k}
\le
f(w_k)
-
\frac{\gamma}{2S}\|\nabla f(w_k)\|^2
+
\frac{\gamma}{2S}\|\nabla f(w_k)-\nabla f(z_k)\|^2
+
\frac{\gamma^2 LG^2}{2}.
\end{equation}
Since $f$ is $L$-smooth, we know that
$
\|\nabla f(w_k)-\nabla f(z_k)\|
\le
L\|w_k-z_k\|.
$
Further, using \Cref{lem:rpd-stale-distance}, we get
$
\|w_k-z_k\|\le \gamma\delta G.
$
Therefore,
\[
\|\nabla f(w_k)-\nabla f(z_k)\|^2
\le
 \gamma^2 A, \qquad A\eqdef L^2\delta^2 G^2.
\]

Substituting this into \eqref{eq:-9=fp9du0f-d9gf} yields
\[
\Exp{f(w_{k+1})\mid\mathcal F_k}
\le
f(w_k)
-
\frac{\gamma}{2S}\|\nabla f(w_k)\|^2
+
\frac{\gamma^3 A}{2S}
+
\frac{\gamma^2 L G^2}{2}.
\]
Taking expectation and  applying the tower property, we get
\[
\Exp{f(w_{k+1})}
\le
\Exp{f(w_k)}
-
\frac{\gamma}{2S}\Exp{\|\nabla f(w_k)\|^2}
+
\frac{\gamma^3 A}{2S}
+
\frac{\gamma^2 L G^2}{2}.
\]
Summing from $k=0$ to $K-1$ gives
\begin{align*}
\Exp{f(w_K)}
&\le
f(w_0)
-
\frac{\gamma}{2S}
\sum_{k=0}^{K-1}\Exp{\|\nabla f(w_k)\|^2}
+
K\frac{\gamma^3 A}{2S}
+
K\frac{\gamma^2 L G^2}{2}.
\end{align*}
Since we know that $\Exp{f(w_K)}\ge f_\star$, we obtain
\[
\frac{\gamma}{2S}
\sum_{k=0}^{K-1}\Exp{\|\nabla f(w_k)\|^2}
\le
f(w_0)-f_\star
+
K\frac{\gamma^3 A}{2S}
+
K\frac{\gamma^2 L G^2}{2}.
\]
Multiply by $2S/(\gamma K)$ to get
\[
\frac1K\sum_{k=0}^{K-1}\Exp{\|\nabla f(w_k)\|^2}
\le
\frac{2S(f(w_0)-f_\star)}{\gamma K}
+
\gamma^2 A
+
\gamma SL G^2,
\]
which establishes \eqref{eq:gpd_random_nonconvex_rate}.
\end{proof}

\begin{corollary}[Stepsize choice]
\label{cor:rpd-tuned-rate-app}Let the assumptions of \Cref{thm:rpd-random-nonconvex-app} hold, fix $K\geq 1$, and choose the stepsize
\[
\gamma
=
\min\!\left\{
\frac1L,\
\sqrt{\frac{2(f(w_0)-f_\star)}{L G^2 K}}
\right\}.
\]
Then
\[
\frac1K\sum_{k=0}^{K-1}\Exp{\|\nabla f(w_k)\|^2}
=
\cO\!\left(\frac{\delta^2}{K}+\frac{S}{\sqrt K}\right).
\]
For fixed $S$ and fixed delay bound $\delta$, this recovers the standard $\cO(K^{-1/2})$ stationarity rate of gradient-type methods for smooth nonconvex optimization.
\end{corollary}

\clearpage

\section{Scaling of delay for PD}
\label{sec:delta-scale-proof}

\begin{figure}[t]
  \centering
  \includegraphics[width=0.95\linewidth]{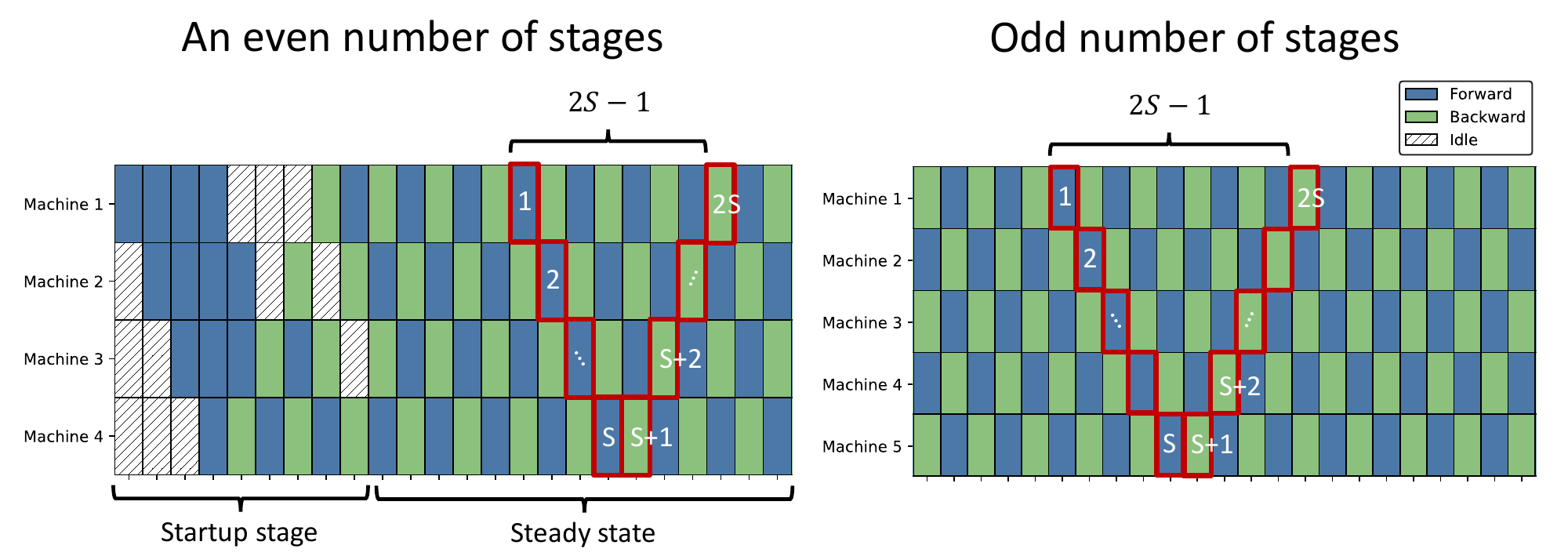}
  \caption{Steady-state counting argument for the exact global-history delay in \algname{PD} 1F1B. The red staircase highlights the forward-to-backward traversal that determines a stale-read window. For a stale block from stage $r$ used at an active backward stage $a$, the window length is $2S-r-a+1$; the worst case is $r=a=1$, with length $2S-1$. In each steady-state pipeline slot, approximately half of the stages perform backward updates: exactly $S/2$ when $S$ is even (left), and alternating between $(S-1)/2$ and $(S+1)/2$ when $S$ is odd (right).}
  \label{fig:proof_delay_1f1b}
\end{figure}

The global-history embedding used in this appendix is the same one used by the simulator. A pipeline slot is first generated as a row of simultaneous stage operations; every backward operation in that row then appends one block update to the global history in stage order. The counting below is exact for this slot-batched linearization. Changing only the within-slot ordering can move individual finite-$S$ delays by lower-order terms, but it does not change the quadratic delay law that drives the convergence comparison.

\begin{lemma}[Steady-state scaling of the exact \algname{PD} delay]
\label{lem:pipedream-delta-scaling}Consider a \algname{PD} {\rm 1F1B} pipeline with $S$ stages, embedded into the global-history \algname{RPD} model. For each steady-state backward update iteration $k$ and each stage $s\in\{1,\dots,S\}$, let $\delta_k^{(s)}$ denote the exact global-history delay of stage $s$ at iteration $k$. Define
\[
\delta_{\mathrm{PD}} \;\eqdef\; \max_{k\in\{1,\dots,K\}} \max_{r\in\{1,\dots,S\}} \delta_k^{(r)},
\qquad
\bar \delta \;\eqdef\; \frac{1}{KS}\sum_{k=1}^K\sum_{r=1}^S \delta_k^{(r)},
\]
where $K$ is the number of steady-state backward-update iterations under consideration.

Under the slot-counting convention used in the global-history embedding, the delay window for stale block $r$ at active stage $a_k$ has length
\[
L_{r,a_k}=2S-r-a_k+1.
\]
Consequently, over a full steady-state period,
\[
\delta_{\mathrm{PD}} = S^2-\frac{S}{2}+O(1),
\qquad
\bar \delta = \frac{S^2}{2}+O(1).
\]
For even $S$, the leading formula is exact:
\[
\delta_k^{(r)}=\frac{S}{2}\bigl(2S-r-a_k+1\bigr),
\qquad
\delta_{\mathrm{PD}}=S^2-\frac{S}{2},
\qquad
\bar\delta=\frac{S^2}{2}.
\]
\end{lemma}

\begin{proof}
We work in the global-history \algname{RPD} embedding of \algname{PD}. Thus, every backward update appends one new entry to the global history, and the delay of a stale block is the number of global-history increments between the forward use of that block and the backward event that consumes the resulting stale model.

Fix a steady-state backward event whose active stage is $a\in\{1,\dots,S\}$, and fix a stale block index $r\in\{1,\dots,S\}$. The microbatch whose gradient is being applied at stage $a$ saw block $r$ during its forward pass on stage $r$. From that forward use, the microbatch must traverse the remaining forward stages $r+1,\dots,S$ and then traverse the backward path from stage $S$ down to stage $a$. Including the active backward slot, this gives the window length
\[
L_{r,a}
=
(S-r)+(S-a+1)
=
2S-r-a+1.
\]
The longest window occurs at $r=a=1$ and has length $2S-1$.

Let $b_t$ denote the number of backward operations performed in steady-state pipeline slot $t$. Since the global history increases by one after each backward update, the delay is obtained by summing $b_t$ over the corresponding window of length $L_{r,a}$.

\medskip
\noindent
\textbf{Case 1: $S$ even.}
In steady state, exactly half of the stages execute backward and half execute forward in every pipeline slot; see the left panel of Figure~\ref{fig:proof_delay_1f1b}. Therefore $b_t=S/2$ for every steady-state slot $t$, and hence
\[
\delta_k^{(r)}
=
\frac{S}{2}L_{r,a_k}
=
\frac{S}{2}\bigl(2S-r-a_k+1\bigr).
\]
Taking $r=1$ and $a_k=1$ gives
\[
\delta_{\mathrm{PD}}
=
\frac{S}{2}(2S-1)
=
S^2-\frac{S}{2}.
\]

To compute the average delay over a full steady-state period, note that each active stage appears equally often. Thus
\begin{align*}
\bar \delta
&=
\frac{1}{S^2}\sum_{a=1}^S\sum_{r=1}^S
\frac{S}{2}\bigl(2S-r-a+1\bigr) \\
&=
\frac{S}{2}
\left(
2S+1-\frac{S+1}{2}-\frac{S+1}{2}
\right)
=
\frac{S^2}{2}.
\end{align*}

\medskip
\noindent
\textbf{Case 2: $S$ odd.}
In steady state, the number of backward operations per slot alternates between
\[
\Bigl\lfloor \frac{S}{2}\Bigr\rfloor = \frac{S-1}{2}
\qquad\text{and}\qquad
\Bigl\lceil \frac{S}{2}\Bigr\rceil = \frac{S+1}{2},
\]
as shown in the right panel of Figure~\ref{fig:proof_delay_1f1b}. Therefore, over any contiguous window of length $L$,
\[
\sum_{t=1}^{L} b_t = \frac{S}{2}L + \varepsilon,
\qquad |\varepsilon|\le \frac12.
\]
Applying this with $L=L_{r,a_k}$ yields
\[
\delta_k^{(r)}
=
\frac{S}{2}\bigl(2S-r-a_k+1\bigr)
+
\varepsilon_k^{(r)}, \quad |\varepsilon_k^{(r)}|\le \frac12.
\]
The maximum again occurs at $r=1$ and $a_k=1$, up to the bounded parity correction, so
\[
\delta_{\mathrm{PD}}
=
S^2-\frac{S}{2}+O(1).
\]
Averaging the leading term over stale blocks and active stages gives $S^2/2$, and the averaged parity corrections remain bounded. Hence
\[
\bar\delta=\frac{S^2}{2}+O(1).
\]

Combining the two parity cases proves the result.
\end{proof}

\begin{corollary}[Quadratic scaling of \algname{PD} staleness]
\label{cor:pipedream-delta-scaling}Consider \algname{PD} {\rm 1F1B} embedded into the global-history \algname{RPD} model. Let
\[
\delta_{\mathrm{PD}} \;\eqdef\; \max_{k\in\{1,\dots,K\}} \max_{s\in\{1,\dots,S\}} \delta_k^{(s)},
\qquad
\bar \delta \;\eqdef\; \frac{1}{KS}\sum_{k=1}^K \sum_{s=1}^S \delta_k^{(s)},
\]
where $\delta_k^{(s)}$ is the exact global-history delay of stale block $s$ at backward-update event $k$, and $K$ is the number of steady-state backward-update iterations under consideration. Then
\[
\delta_{\mathrm{PD}} = S^2-\frac{S}{2}+O(1),
\qquad
\bar \delta = \frac{S^2}{2}+O(1).
\]
\end{corollary}

\begin{proof}
This follows directly from Lemma~\ref{lem:pipedream-delta-scaling}. For even $S$, the identities
\[
\delta_{\mathrm{PD}} = S^2-\frac{S}{2},
\qquad
\bar \delta = \frac{S^2}{2}
\]
hold exactly. For odd $S$, Lemma~\ref{lem:pipedream-delta-scaling} gives
\[
\delta_{\mathrm{PD}} = S^2-\frac{S}{2}+O(1),
\qquad
\bar \delta = \frac{S^2}{2}+O(1).
\]
Combining the two cases yields the claim.
\end{proof}

\begin{figure}[t]
  \centering
  \includegraphics[width=0.95\linewidth]{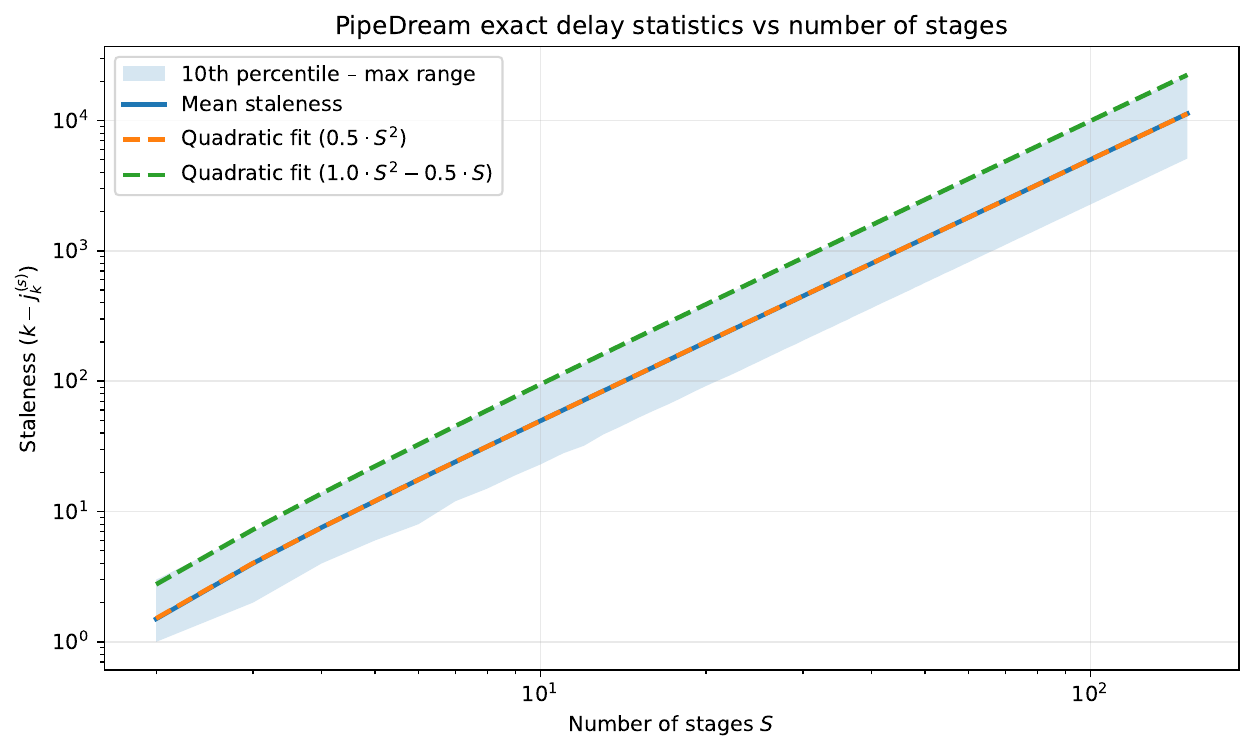}
  \caption{Steady-state statistics of the exact global-history delay $\delta_k^{(s)}$ induced by \algname{PD} {\rm 1F1B} as a function of the number of pipeline stages $S$. For each value of $S$, we generated the corresponding {\rm 1F1B} timeline, reconstructed the exact delays $\delta_k^{(s)}$ directly from the timeline, discarded the first half of the resulting delay sequence to remove the startup transient, and then computed statistics on the remaining steady-state part. The solid blue curve shows the empirical mean delay $\bar \delta$, while the shaded region spans from the $10$th percentile to the maximum delay. Both statistics exhibit quadratic growth in $S$. The orange dashed curve corresponds to the fit $\bar \delta(S)\approx \tfrac12 S^2$, while the green dashed curve corresponds to $\delta(S)\approx S^2-\tfrac12 S$, where $\delta=\max_k\max_s \delta_k^{(s)}$.}
  \label{fig:pipedream_delay_scaling}
\end{figure}

To validate the asymptotic scaling predicted by Lemma~\ref{lem:pipedream-delta-scaling}, we computed the exact global-history delays induced by \algname{PD} {\rm 1F1B} over a wide range of pipeline depths. For each stage count $S$, we first generated the corresponding {\rm 1F1B} schedule and then reconstructed the exact delay matrix directly from the timeline. More precisely, for every steady-state backward-update iteration $k$ and every stage $s$, we computed
\[
\delta_k^{(s)} = k - j_k^{(s)},
\]
where $j_k^{(s)}$ is the global history index of the snapshot used by the corresponding microbatch during its forward pass on stage $s$. Since our goal is to measure the steady-state behavior rather than the startup transient, we discarded the first half of the resulting exact-delay sequence and analyzed only the remaining part, where the pipeline has entered its periodic {\rm 1F1B} regime.

Figure~\ref{fig:pipedream_delay_scaling} reports the resulting statistics as functions of $S$. The solid blue curve plots the empirical average delay
\[
\bar \delta = \frac{1}{KS}\sum_{k=1}^K\sum_{s=1}^S \delta_k^{(s)},
\]
while the shaded region shows the spread from the $10$th percentile to the largest observed delay. The two dashed curves correspond to the quadratic laws predicted by the steady-state counting argument:
\[
\bar \delta(S) \approx \frac{S^2}{2} = \bar \delta,
\qquad
\delta(S) \approx S^2 - \frac{S}{2} = \delta_{\mathrm{PD}},
\]
where
\[
\delta_{\mathrm{PD}} = \max_{k\in\{1,\dots,K\}} \max_{s\in\{1,\dots,S\}} \delta_k^{(s)}.
\]
The agreement is remarkably tight. In particular, the average delay grows as $\tfrac12 S^2$, while the worst-case delay grows as $S^2-\tfrac12 S$, confirming that the exact global-history staleness induced by \algname{PD} scales quadratically with the number of stages.

This observation is important for the \algname{RPD} analysis. In our theoretical convergence bound, the staleness parameter $\delta$ appears explicitly, so a quadratic growth of $\delta$ with $S$ implies a corresponding degradation of the theoretical guarantee as the pipeline depth increases. Thus, even though \algname{PD} {\rm 1F1B} achieves high hardware utilization, in the global-history \algname{RPD} abstraction it induces a staleness regime that becomes progressively less favorable for optimization as $S$ grows. This motivates the search for alternative scheduling policies that preserve good utilization while reducing the effective global-history delay.

\clearpage

\section{Derivation of the LocalSGD rate under the block-update budget}
\label{sec:localsgd-rate-derivation}

We study the following variation of the \algname{LocalSGD} algorithm for solving (\ref{eq:objective}):

\begin{algorithm}[H]
\caption{Stage-distributed \algname{LocalSGD}}
\label{alg:localsgd-1f1b}
\begin{algorithmic}[1]
\STATE Input: stepsize $\gamma>0$, number of stages $S$, number of replicas $R$,
synchronization period $H$, initial model
$w_0=(w_0^{(1)},\dots,w_0^{(S)})$
\STATE Initialize replicas: $w_{0,r}^{(s)} = w_0^{(s)} \qquad \forall r\in\{1,\dots,R\},\quad \forall s\in\{1,\dots,S\}$.
\FOR{$q=0,1,\dots$}
    \FOR{each replica $r=1,\dots,R$}
        \STATE Sample $m_{q,r}\in\{1,\dots,M\}$.
        \STATE Run a 1F1B pass for replica $r$ on $m_{q,r}$ across $S$ stages.
        \STATE Update replica $r$: $w_{q+1,r}^{(s)}=w_{q,r}^{(s)}-\gamma \nabla_{w^{(s)}} f_{m_{q,r}}(w_{q,r}) \qquad \forall s\in\{1,\dots,S\}.$
    \ENDFOR
    \IF{$q+1$ is divisible by $H$}
        \STATE Average weights between replicas: $\bar w_{q+1}^{(s)} = \frac1R\sum_{r=1}^R w_{q+1,r}^{(s)} \qquad \forall s\in\{1,\dots,S\}.$
        \STATE Synchronize replicas: $w_{q+1,r}^{(s)} \leftarrow \bar w_{q+1}^{(s)} \qquad \forall r\in\{1,\dots,R\},\quad\forall s\in\{1,\dots,S\}.$
    \ENDIF
\ENDFOR
\end{algorithmic}
\end{algorithm}

We use the standard nonconvex \algname{LocalSGD} analysis of
\citet{yu2019parallel} in the homogeneous-worker case. To avoid overloading the
finite-sum components $f_m$ from \eqref{eq:objective}, denote the objective
assigned to replica $r$ in the standard LocalSGD formulation by $\phi_r$.
In our setting,
\[
\phi_1=\cdots=\phi_R=f,
\]
where $R$ is the number of replicas, i.e., the number of parallel local
trajectories. Thus, the worker count in \citet{yu2019parallel} corresponds to
our $R$, not to the finite-sum size $M$ in \eqref{eq:objective}.

Let
\[
\bar w_q
:=
\frac1R\sum_{r=1}^R w_{q,r}
\]
be the averaged logical model after $q$ local steps per replica. Assume that
$f$ is lower bounded by $f_\star$ and $L$-smooth. For uniformly sampled
mini-batches $m_{q,r}\in\{1,\dots,M\}$, assume
\[
\E\!\left[
\nabla f_{m_{q,r}}(w_{q,r})
\mid w_{q,r}
\right]
=
\nabla f(w_{q,r}),
\]
and
\[
\E\!\left[
\left\|
\nabla f_{m_{q,r}}(w_{q,r})
-
\nabla f(w_{q,r})
\right\|^2
\right]
\le \sigma^2,
\qquad
\E\!\left[
\left\|
\nabla f_{m_{q,r}}(w_{q,r})
\right\|^2
\right]
\le B^2.
\]
We use $B$ for the second-moment bound in order to avoid confusion with the
block-gradient bound $G$ used in the \algname{RPD} analysis.

The standard periodic-averaging \algname{LocalSGD} guarantee gives that, for
$0<\gamma\le 1/L$ and $T$ local steps per replica,
\begin{equation}
\label{eq:localsgd-rate-T-app}
\frac{1}{T}\sum_{q=0}^{T-1}
\E\bigl[\|\nabla f(\bar w_q)\|^2\bigr]
\le
\frac{2\bigl(f(\bar w_0)-f_\star\bigr)}{\gamma T}
+
4\gamma^2H^2B^2L^2
+
\frac{L\gamma\sigma^2}{R}.
\end{equation}

In our staged implementation, one logical local SGD step updates all $S$ blocks
of one replica. Since there are $R$ replicas, $T$ local steps per replica
correspond to
\[
K=RST
\]
stage/block updates. Equivalently,
\[
T=\frac{K}{RS}.
\]
Substituting this relation into \eqref{eq:localsgd-rate-T-app} gives the bound
under the common block-update budget:
\begin{equation}
\label{eq:localsgd-rate-K-app}
\frac{RS}{K}\sum_{q=0}^{K/(RS)-1}
\E\bigl[\|\nabla f(\bar w_q)\|^2\bigr]
\le
\frac{2RS\bigl(f(\bar w_0)-f_\star\bigr)}{\gamma K}
+
4\gamma^2H^2B^2L^2
+
\frac{L\gamma\sigma^2}{R}.
\end{equation}

Using the stepsize choice from the corollary of \citet{yu2019parallel},
\[
\gamma=\frac{\sqrt{R}}{L\sqrt{T}},
\]
truncated at $1/L$ when necessary, balances the optimization term and the
stochastic-variance term in \eqref{eq:localsgd-rate-T-app} and yields
\[
\frac{1}{T}\sum_{q=0}^{T-1}
\E\bigl[\|\nabla f(\bar w_q)\|^2\bigr]
=
\cO\!\left(
\frac{1}{\sqrt{RT}}
+
\frac{RH^2}{T}
\right).
\]
Their linear-speedup corollary additionally requires
$H\le T^{1/4}/R^{3/4}$ to make the local-drift contribution no larger than the
$1/\sqrt{RT}$ term. We do not impose this restriction in the displayed
comparison; instead we retain the explicit $H$-dependent term.

Finally, using $T=K/(RS)$, we obtain
\[
\frac{1}{\sqrt{RT}}
=
\sqrt{\frac{S}{K}},
\qquad
\frac{RH^2}{T}
=
\frac{R^2SH^2}{K}.
\]
Therefore,
\begin{equation}
\label{eq:localsgd-tuned-rate-K-app}
\frac{RS}{K}\sum_{q=0}^{K/(RS)-1}
\E\bigl[\|\nabla f(\bar w_q)\|^2\bigr]
=
\cO\!\left(
\sqrt{\frac{S}{K}}
+
\frac{R^2SH^2}{K}
\right).
\end{equation}
For the experimental choice $R=S$, this reduces to
\[
\frac{RS}{K}\sum_{q=0}^{K/(RS)-1}
\E\bigl[\|\nabla f(\bar w_q)\|^2\bigr]
=
\cO\!\left(
\sqrt{\frac{S}{K}}
+
\frac{S^3H^2}{K}
\right).
\]

\clearpage

\section{Discrete 1F1B Scheduling Simulators}
\label{sec:1f1b-scheduling}

This appendix describes the discrete-event schedulers used to generate the
pipeline timelines in our experiments. The purpose of these schedulers is to
separate the systems-level ordering of forward and backward computations from
the optimization dynamics. We first generate a static timeline of pipeline
actions and then replay this timeline during optimization. This makes the
comparison between \algname{PD} and \algname{LocalSGD} independent of low-level
runtime effects such as kernel launch overheads, communication latency, or
device-specific implementation details.

We consider a straight pipeline with $S$ stages. Time is discrete. At every
time step, each stage can execute at most one action:
\[
a_{t,s} \in \{\emptyset\} \cup \{F_m : m \in \mathcal M\}
                    \cup \{B_m : m \in \mathcal M\},
\]
where $F_m$ denotes the forward computation of microbatch $m$, $B_m$ denotes
the corresponding backward computation, and $\emptyset$ denotes an idle slot.
The complete schedule is therefore a timeline
\[
\mathcal T = \bigl(a_{t,1},\dots,a_{t,S}\bigr)_{t\ge 0}.
\]
In all schedulers, data movement is strictly causal: a computation that depends
on another stage may start only at a later discrete time step, not in the same
slot.

\subsection{PD 1F1B schedule}
\label{sec:pd-1f1b-schedule}

The \algname{PD} scheduler follows a one-forward-one-backward (1F1B) policy
after pipeline startup. The input stage admits at most $A$ active microbatches
into the pipeline, where $A$ is the number of outstanding active microbatches
allowed by the schedule. In our experiments we use the standard one-machine-per-
stage choice $A=S$, corresponding to the NOAM (number of outstanding active
microbatches) parameter in the original \algname{PD} implementation.

For each stage $s$ and microbatch $m$, we store two completion times:
\[
\tau^F_{s,m}, \qquad \tau^B_{s,m},
\]
where $\tau^F_{s,m}$ is the time at which the forward computation of microbatch
$m$ finishes at stage $s$, and $\tau^B_{s,m}$ is the time at which its backward
computation finishes at stage $s$. A value of $-1$ means that the corresponding
operation has not yet been completed.

At time $t$, a forward operation $F_m$ is ready at stage $s$ if
\[
\begin{cases}
\begin{aligned}[t]
& m \text{ is the next unlaunched microbatch,}\\
& \text{and the number of active microbatches is less than } A,
\end{aligned}
& \text{if } s=1,\\[0.5em]
\tau^F_{s-1,m} \neq -1 \text{ and } \tau^F_{s-1,m} < t,
& \text{if } s>1.
\end{cases}
\]
A backward operation $B_m$ is ready at stage $s$ if the forward computation has
already completed at that stage and either $s$ is the last stage or the backward
computation has already completed at the next stage:
\[
\tau^F_{s,m} \neq -1, \qquad \tau^F_{s,m}<t,
\]
and
\[
\begin{cases}
\text{true}, & s=S,\\
\tau^B_{s+1,m} \neq -1 \text{ and } \tau^B_{s+1,m}<t, & s<S.
\end{cases}
\]
The strict inequalities enforce that activations and gradients cannot propagate
through multiple stages in a single discrete slot.

Each stage begins in a startup phase. During startup, the stage prioritizes
forward work until the first backward operation becomes ready. Once a backward
operation becomes available, the stage enters steady state. In steady state, the
stage alternates between backward and forward work whenever both are ready. If
only one type of work is ready, it executes that available operation rather than
idling.

Equivalently, each stage stores a local preference variable
\[
p_s \in \{F,B\}.
\]
Initially $p_s=F$. After executing a backward operation, the stage sets
$p_s=F$; after executing a forward operation in steady state, it sets $p_s=B$.
Thus the local decision rule is:
\[
a_{t,s}
=
\begin{cases}
B_m, &
\begin{aligned}[t]
&\text{if backward work is ready and either}\\
&\text{the stage is not yet in steady state,}\\
& p_s = B,\ \text{or no forward work is ready},
\end{aligned}
\\[0.4em]
F_m, & \text{if forward work is ready},\\
\emptyset, & \text{otherwise}.
\end{cases}
\]
After all stages have selected their actions for time $t$, the completion-time
arrays are updated simultaneously. This simultaneous update is important: it
prevents one stage from using information produced by another stage in the same
time slot.

The active-microbatch counter is updated only at the input stage. A microbatch
becomes active when its forward computation at stage $1$ completes, and it stops
being active when its backward computation at stage $1$ completes. The schedule
terminates when the final microbatch has completed its backward computation at
stage $1$.

\subsection{LocalSGD 1F1B schedule}
\label{sec:localsgd-1f1b-schedule}

The \algname{LocalSGD} scheduler uses the same stage-level 1F1B execution rule,
but the admissible forward operations are constrained by the local-training and
synchronization structure. Instead of one pipeline trajectory, we run $R$
independent local trajectories, which we call runs or replicas. Each
synchronization round consists of $H$ local steps for each replica. Therefore one
round contains
\[
RH
\]
pipeline microbatch jobs.

We index the scheduled pipeline jobs by a single integer $m=0,1,\dots,N-1$.
This index can be decoded as
\[
q(m) = \left\lfloor \frac{m}{RH} \right\rfloor,
\qquad
\ell(m) =
\left\lfloor \frac{m \bmod RH}{R} \right\rfloor,
\qquad
r(m) = m \bmod R,
\]
where $q(m)$ is the synchronization round, $\ell(m)\in\{0,\dots,H-1\}$ is the
local step inside that round, and $r(m)\in\{0,\dots,R-1\}$ is the replica index.
Thus the microbatches in a round are ordered as
\[
\begin{aligned}
&
(r=0,\ell=0),\dots,(r=R-1,\ell=0),
\\
&
(r=0,\ell=1),\dots,(r=R-1,\ell=1),
\dots,
(r=R-1,\ell=H-1).
\end{aligned}
\]

The readiness rule for backward operations is identical to the \algname{PD}
scheduler. The difference is in the readiness rule for forward operations.
A forward operation $F_m$ at stage $s$ must satisfy three conditions.

First, the usual pipeline dependency must hold:
\[
\begin{cases}
\text{true}, & s=1,\\
\tau^F_{s-1,m}\neq -1 \text{ and } \tau^F_{s-1,m}<t, & s>1.
\end{cases}
\]

Second, if $m$ is not the first local step of its replica inside the current
round, then the previous local step of the same replica must already have
completed its backward computation at the same stage. Since replicas are ordered
fastest within each local-step group, the previous local step of the same
replica has index $m-R$. Hence, for $\ell(m)>0$, we require
\[
\tau^B_{s,m-R}\neq -1,
\qquad
\tau^B_{s,m-R}<t.
\]
This condition enforces that a stage does not start the next local step of a
replica before its own block has been updated by the previous local step.

Third, if $m$ belongs to a new synchronization round, then the previous round
must be fully complete before any microbatch from the new round is launched. Let
$q(m)>0$. The last microbatch of the previous round is
\[
m_{\mathrm{last}}(q(m)-1) = q(m)RH - 1.
\]
Since the backward pass finishes at stage $1$, it is sufficient to check
\[
\tau^B_{1,\,q(m)RH-1} \neq -1,
\qquad
\tau^B_{1,\,q(m)RH-1}<t.
\]
This implements the global model-averaging barrier between \algname{LocalSGD} rounds.
Only after all replicas have completed their $H$ local steps is the averaged
model broadcast and the next round allowed to begin.

Once the forward and backward readiness sets have been constructed, the same
local 1F1B choice rule as in \algname{PD} is applied at every stage. Thus
\algname{LocalSGD} differs from \algname{PD} not in the local stage policy, but
in the admissibility constraints for launching forward work. These constraints
encode two optimization-level requirements: consecutive local steps of the same
replica must respect the stage-local update order, and different synchronization
rounds must be separated by a global averaging barrier.

\subsection{Comparison of the two timelines}

The two schedules share the same pipeline mechanics: forward computations move
from stage $1$ to stage $S$, backward computations move from stage $S$ to stage
$1$, and each stage follows a 1F1B preference once backward work becomes
available. The difference is the source of idle time.

In \algname{PD}, idle time appears mainly during startup and drain. Once the
pipeline is filled, the NOAM constraint allows many microbatches to remain active
simultaneously, so the schedule has high utilization. The cost is optimization
staleness: a backward computation may update a stage using weights and
activations generated many pipeline slots earlier.

In \algname{LocalSGD}, the schedule has additional idle regions at
synchronization boundaries. No microbatch from round $q+1$ is allowed to start
until all microbatches from round $q$ have completed their backward pass at the
input stage. This creates synchronization bubbles. However, the method avoids
the same long cross-stage weight-version staleness mechanism that appears in
\algname{PD}; instead, its optimization error comes from local replica drift
between averaging events.

This is the scheduling-level origin of the main tradeoff studied in the paper:
\algname{PD} maximizes pipeline utilization by accepting stale pipeline updates,
whereas \algname{LocalSGD} reduces stale-version effects by periodically paying
for synchronization barriers.

\begin{figure}[t]
  \centering
  \includegraphics[width=0.95\linewidth]{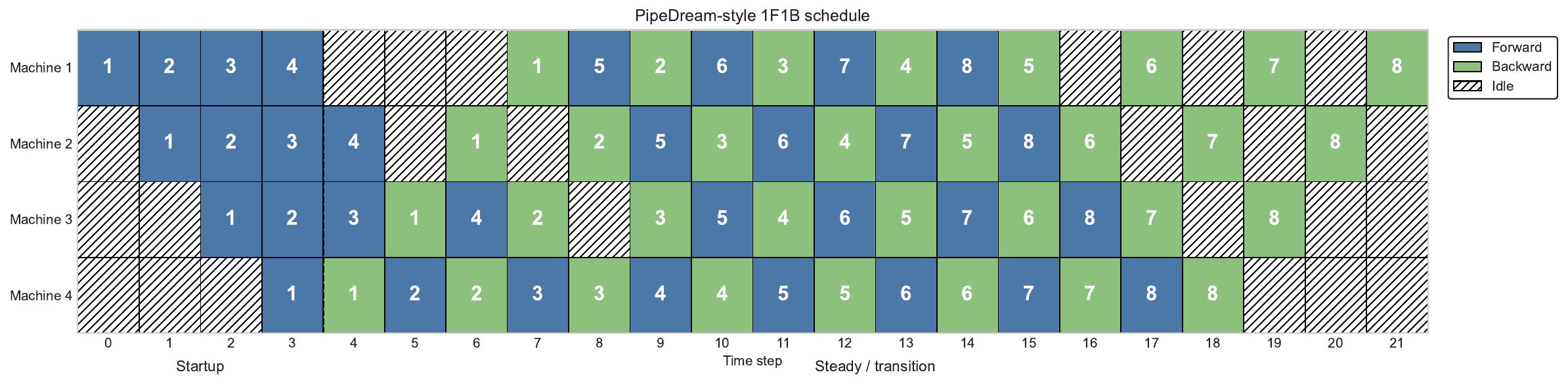}
  \caption{\algname{PD}-style 1F1B schedule for a 4-stage pipeline with 8 microbatches. Rows correspond to pipeline stages (machines), and columns correspond to discrete time steps. Each numbered cell indicates the microbatch being processed at that stage and time. The figure illustrates how \algname{PD} overlaps computation across stages to improve utilization after the initial fill phase.}
  \label{fig:pipedream-schedule}
\end{figure}

\clearpage

\section{Relationship between PD and RPD}

To understand the connection between \algname{PD} and our abstract \algname{RPD} method (\Cref{alg:gpd}), it is useful to distinguish between two notions of staleness that are easily conflated in implementation.

In \algname{PD}, the model is partitioned into $S$ stages, and each stage updates its own block of parameters asynchronously according to the pipeline schedule. When a microbatch performs its forward pass through stage $s$, the current weights of that stage are \emph{stashed}, so that the corresponding backward pass later uses exactly the same stage-local weights. Thus, each microbatch carries a stage-wise record of the parameter values that it encountered during the forward pass. The resulting stale model seen by a given backward event is therefore not arbitrary: it is the collection of stage-wise weights that were actually used when that same microbatch traversed the pipeline on the forward pass.

By contrast, in \algname{RPD}, the update rule is written in terms of a stale mixed model
\[
z_k = \bigl(w_{j_k^{(1)}}^{(1)}, \dots, w_{j_k^{(S)}}^{(S)}\bigr),
\]
where $j_k^{(s)} \le k$ specifies which historical version of block $s$ is read at iteration $k$. The crucial observation is that \algname{RPD} does not require all blocks to be read from the same past iterate. Different blocks may be taken from different history indices, which makes it expressive enough to represent stage-wise stale reads similar to those induced by pipeline parallelism.

\paragraph{Deterministic \algname{PD} as a structured stale-block instance.}
Consider the sequence of backward-update events in \algname{PD}, flattened in their actual execution order:
\[
e_k = (s_k, m_k), \qquad k = 0,1,2,\dots,
\]
where $s_k$ is the stage updated at the $k$-th backward event and $m_k$ is the corresponding microbatch. If, for each such event, one defines $z_k$ to be exactly the collection of stage-wise weights that microbatch $m_k$ saw on the forward pass in \algname{PD}, then the resulting \algname{RPD} update coincides with the \algname{PD} update. In this sense, \algname{PD} can be viewed as a structured special case of the stale mixed-model update template, in which the arriving stages, selected batches, and stale block reads are all determined by the pipeline schedule and by the forward-pass stashing mechanism.

This equivalence is at the level of the update template. It does not mean that deterministic \algname{PD} satisfies the uniform random stage/mini-batch sampling assumption used in \Cref{thm:gpd}; that theorem applies to the randomized abstraction, and the deterministic scheduler enters the paper through the delay calculation.

This viewpoint is important because it clarifies what must be matched in order for \algname{RPD} to reproduce \algname{PD} exactly. Matching only the order of stage updates is not sufficient. Matching only the batch order is not sufficient. Even matching a simple stage-dependent delay pattern is not sufficient. Exact equivalence requires that, for each backward event, the stale model $z_k$ contain the same stage-wise parameter values that the corresponding microbatch encountered on the forward pass in \algname{PD}.

\paragraph{Global history versus local staleness.}
Our implementation maintains a global history
\[
w_0, w_1, w_2, \dots,
\]
where each new iterate is created after a single block update. Since only one block changes per update, the $k$-th history element stores the full parameter vector after the $k$-th global block update. This global history should be distinguished from the \emph{local} version counters maintained by each stage in \algname{PD}.

This distinction leads to two different notions of delay.

First, one may define the \emph{local stage staleness} for microbatch $m$ and stage $s$ as
\[
\Delta^{\mathrm{local}}(m,s)
=
v_s^{\mathrm{current}} - v_s^{\mathrm{forward}}(m),
\]
where $v_s^{\mathrm{forward}}(m)$ is the local version of stage $s$ used by microbatch $m$ during its forward pass, and $v_s^{\mathrm{current}}$ is the current local version of stage $s$ at the time of the backward update. This is the natural stage-local notion of staleness in \algname{PD}.

Second, when representing \algname{PD} inside \algname{RPD}, one must work with a \emph{global history delay}. For the backward event $e_k=(s_k,m_k)$, let
\[
j_k^{(s)}
\]
denote the global history index at which microbatch $m_k$ saw stage $s$ during its forward pass. Then the exact delay needed by \algname{RPD} is
\[
\delta_k^{(s)} = k - j_k^{(s)},
\]
so that
\[
w_{j_k^{(s)}}^{(s)} = w_{k-\delta_k^{(s)}}^{(s)}.
\]
In other words, the delay used by \algname{RPD} is measured with respect to the \emph{global} sequence of block updates, not with respect to the local update count of stage $s$.

\paragraph{Why can the exact delays be large?}
A potentially confusing phenomenon is that the exact delay for a stage can be much larger than its local staleness, and can even be strictly positive for the final stage in steady state. This is not a contradiction. The reason is that the global history index advances after \emph{every} block update, regardless of which stage was updated. Thus, even if a particular stage has not changed much locally, many global history entries may have been inserted by updates to other stages.

For example, suppose a microbatch reaches the final stage on the forward pass, and before its corresponding backward event occurs, several other stages perform backward updates. Even if the final-stage weights themselves are nearly unchanged in a local sense, the required forward-time snapshot for that stage may now lie several entries back in the global history. Consequently, the exact global delay can be larger than zero, and in practice may be substantially larger than the number of local updates experienced by that stage.

This explains an empirical phenomenon observed in our experiments: in the steady regime, the exact delay vectors extracted from \algname{PD} may contain relatively large entries, even for stages that are intuitively ``fresh.'' Such values should not be interpreted as large local staleness of that stage; rather, they represent how far one must rewind the \emph{global} history in order to recover the exact forward-time block value used by the corresponding microbatch.

\paragraph{Steady-state structure.}
Once the pipeline reaches steady state, the exact delay vectors induced by \algname{PD} exhibit a periodic pattern. This is a direct consequence of the deterministic structure of the $1\mathrm{F}1\mathrm{B}$ schedule. For a fixed number of stages, the sequence of backward-update events repeats cyclically, and therefore the associated exact stale-read vectors repeat as well. This steady-state periodicity is useful both conceptually and practically: conceptually, it shows that pipeline-induced staleness is highly structured rather than arbitrary; practically, it provides a way to construct \algname{RPD} instances that reproduce \algname{PD} exactly or approximately by feeding in the corresponding schedule-induced batch order, stage order, and exact delay vectors.

\paragraph{Practical consequence.}
These observations suggest three levels of approximation when relating \algname{RPD} to \algname{PD}. At the coarsest level, one may use simple synthetic delay rules, such as stage-dependent offsets, to obtain a rough pipeline-inspired stale model. At an intermediate level, one may match the stage-update order and batch-selection order to those of \algname{PD}. At the finest level, one may additionally use the exact delay vectors extracted from a \algname{PD} run, in which case \algname{RPD} reproduces the same stale model seen by each backward event. This hierarchy provides a principled way to move from a mathematically convenient abstraction toward a faithful simulation of the original pipeline method.

\clearpage

\section{Experiments and setup}
\label{sec:experiments-extra}

\subsection{Compute Resources}
\label{sec:compute-resources}

We report the compute resources used to run the experiments in this paper. The synthetic experiments, including the quadratic objectives and logistic-regression objectives, were run on a MacBook Pro equipped with an Apple M1 Max chip and 64GB of RAM. These experiments were implemented in Python 3.12.

For the larger LLM experiments, we used a single NVIDIA A100 GPU with 40GB of VRAM. These experiments were also implemented in Python 3.12.

The synthetic experiments are lightweight and can be reproduced on a modern laptop or workstation with sufficient memory. The LLM experiments require GPU acceleration; in our setup, one NVIDIA A100 40GB GPU was sufficient.

\subsection{Objectives}

\paragraph{Experimental objectives.}
We use synthetic finite-sum objectives of the form \eqref{eq:objective}. In all cases, the parameter vector $w\in\R^d$ is partitioned into $S$ contiguous blocks, corresponding to the $S$ pipeline stages. The data are split into mini-batches, and each $f_m$ is the loss evaluated on the $m^{\text{th}}$ mini-batch.

\paragraph{Quadratic objective.}
Our first benchmark is a least-squares quadratic objective. We generate $X\in\R^{N\times d}$, a ground-truth vector $w^\star\in\R^d$, and targets
\[
y = Xw^\star
\]
with optional additive label noise. The objective is
\[
f(w)
=
\frac{1}{2N}\|Xw-y\|_2^2.
\]
Equivalently, this is a quadratic with Hessian
\[
A = \frac{1}{N}X^\top X.
\]
We use two variants. In the random quadratic, $X$ has i.i.d. standard Gaussian entries and $w^\star$ is also sampled from a standard Gaussian distribution. In the tridiagonal quadratic, we construct a Nesterov-style tridiagonal matrix
\[
T
=
\begin{pmatrix}
2 & -1 &        &        \\
-1 & 2 & -1     &        \\
   & \ddots & \ddots & \ddots \\
   &        & -1 & 2
\end{pmatrix},
\]
add a diagonal shift $A=T+\mu I$, and then construct $X$ so that
\[
\frac{1}{N}X^\top X = A
\]
holds exactly. When a target condition number is specified, $\mu$ is chosen so that the resulting matrix $A$ has that condition number; otherwise we use the default shift from the implementation. The optimal value $f_\star$ for this objective is computed by solving the corresponding least-squares problem.

The \algname{PD}/\algname{RPD} validation plot in \Cref{fig:rpd-validation-and-scaling} uses the random quadratic variant. The quadratic panel in \Cref{fig:pd-vs-localsgd} uses the Nesterov-style tridiagonal quadratic variant.

\paragraph{Logistic-regression objective.}
Our second benchmark is binary logistic regression. We sample $X\in\R^{N\times d}$ with i.i.d. standard Gaussian entries and generate a ground-truth vector $w^\star\in\R^d$. Labels are sampled according to
\[
\mathbb{P}(y_i=1\mid x_i)
=
\sigma(x_i^\top w^\star),
\qquad
\sigma(t)=\frac{1}{1+\exp(-t)}.
\]
The optimized objective is the regularized binary cross-entropy
\[
f(w)
=
\frac{1}{N}\sum_{i=1}^{N}
\left[
\max(x_i^\top w,0)
-
y_i x_i^\top w
+
\log\!\left(1+\exp(-|x_i^\top w|)\right)
\right]
+
\frac{\lambda}{2}\|w\|_2^2,
\]
where $\lambda\ge 0$ is the $\ell_2$-regularization parameter. The implementation uses the numerically stable form above for both full-objective evaluation and mini-batch losses. For plots reported in objective-gap form, the reference value $f_\star$ is computed numerically using L-BFGS-B with the exact gradient.

Unless stated otherwise, we use $d=512$ in the reported experiments. The number of stages $S$ is varied in the scaling experiments and fixed to the value indicated in each figure caption for the direct method comparisons. Each method is tuned separately for the corresponding objective and stage count.

\subsection{Weights stashing}

Once the static timeline has been generated, we replay it to simulate training. Each stage maintains its current parameter vector and a local version counter. When a forward operation for microbatch $m$ is executed at stage $i$, the stage uses the most recent locally available weights $w_i^{(t)}$. Immediately after this forward computation, that exact weight value and its version index are stored as part of the microbatch state for stage $i$. Later, when the matching backward operation for the same microbatch and stage is executed, the gradient is computed using the stashed version rather than the current version. After the gradient has been formed, the stage performs an SGD block update
\[
w^{(s_k)} \leftarrow w^{(s_k)} - \gamma \nabla_{w^{(s_k)}}\ell_b,
\]
and increments its local version counter.

This is exactly the \algname{PD} weight-stashing rule: forward uses the latest available stage-local weights, and backward for a given microbatch reuses the same stage-local version that was used during forward. Importantly, this guarantees consistency only within a stage. It does not force all stages to use the same global parameter version for the same microbatch. \algname{PD} calls the stronger cross-stage consistency mechanism vertical sync, but that is not part of its default semantics, and we did not include it in our initial simulator.

\subsection{What was recorded}

For reproducibility, we tracked four classes of quantities during replay.

First, we stored the full operation timeline, i.e. for each time step and each stage, whether the executed action was forward, backward, or idle, and for which microbatch. This timeline fully specifies the schedule.

Second, for every microbatch-stage pair, we stored the weight version used during forward and the weight version reused during backward. This allowed us to verify that the weight-stashing invariant held exactly:
\[
\text{forward\_version}(m,i) = \text{backward\_version}(m,i).
\]

Third, we measured stage-local staleness, defined as the difference between the current local version at backward time and the stashed version that had been used at forward time. This quantity captures how many local updates occurred at a stage between the forward and backward computations of the same microbatch.

Fourth, we tracked the full-dataset objective value over simulated pipeline time, as well as after each microbatch completed its backward pass at stage $1$. The first view measures optimization progress against a wall-clock proxy given by pipeline time steps; the second measures optimization progress against the number of completed training examples.

\subsection{Scope of the current simulator}

The simulator is intentionally minimal. We assume one worker per stage, ignore communication delay, and treat each forward and backward stage operation as occupying one discrete time slot. We also omit vertical sync. These choices are deliberate: the goal of the experiments is not to build a full systems simulator, but to isolate the effect of \algname{PD}'s scheduling and weight-stashing semantics on optimization in a setting where the ground-truth objective is simple and fully observable.

The simulator also does not charge memory as a separate axis in the plots. In particular, \algname{LocalSGD}'s $R$ logical replicas require additional model storage compared with the single \algname{PD} trajectory, and this extra memory is intentionally reported as a modeling difference rather than folded into the simulated-time budget.

\subsection{Reproducibility details for the synthetic experiments}
\label{sec:synthetic-repro-details}

All synthetic experiments use zero initialization and seed $0$ for the reported runs, with the deterministic training batch order generated from seed $1$. When a method has internal sampling randomness, the same reported seed is passed to that method. Unless stated otherwise, the learning rate is selected by a grid sweep and the plotted curve is the stable selected run for that method/configuration. The main-paper synthetic plots are single-seed runs and therefore do not display error bars; the sweep utilities in the repository also support multi-seed aggregation and standard-deviation summaries.

For the \algname{PD}/\algname{RPD} validation in \Cref{fig:rpd-validation-and-scaling}, we use the random quadratic objective with $S=8$, $d=512$, batch size $10$, $M=60$ data batches, and $5$ epochs, so the \algname{PD} timeline contains $300$ microbatches. The learning-rate grid for both \algname{PD} and \algname{RPD} is
\[
\{2^{-8},2^{-7},\dots,2^0\}.
\]
The delay sweep shown with \algname{RPD} uses
\[
\delta\in\{8,30,60,180,420\},
\]
where $\delta=60$ is the predicted steady-state \algname{PD} value $S^2-S/2$ for $S=8$.

For the fixed-simulator-time logistic-regression scaling experiment in \Cref{fig:fixed-time-scaling}, we use $d=512$, batch size $10$, $M=60$, $\lambda=10^{-4}$, gradient-noise standard deviation $0.5$, and stage counts
\[
S\in\{2,4,8,16,32,64,128\}.
\]
The common simulator-time target is $3684$ discrete pipeline ticks, obtained as six times the $S=8$, $300$-microbatch \algname{PD} timeline length used by the validation experiment. For each $S$, the \algname{PD} and \algname{LocalSGD} timelines use the smallest number of microbatches whose schedule reaches this target. \algname{RPD} is assigned the number of block updates equal to the number of backward operations completed by the matched \algname{PD} timeline. The learning-rate grids are $\{2^{-12},2^{-11},\dots,2^{-3}\}$ for \algname{PD}/\algname{RPD} and $\{2^{-8},2^{-7},\dots,2^2\}$ for \algname{LocalSGD}.
For this scaling plot, \algname{LocalSGD} uses $R=S$ and $H=5$.

For the $S=16$ \algname{PD} versus \algname{LocalSGD} panels in \Cref{fig:pd-vs-localsgd}, we use the same simulator and learning-rate selection rule. The synthetic runs use $d=512$ and zero initialization. \algname{LocalSGD} uses $R=S$; the synchronization period $H$ is swept in the corresponding notebook and then the displayed panel reports the selected tuned configuration for the plotted objective.

\subsection{Experimental setup.}
We compare \algname{PD} and \algname{RPD} on the same synthetic quadratic objective under a fixed block-update budget. For \algname{PD}, we use the standard {\rm 1F1B} schedule. For \algname{RPD}, we run separate experiments for several values of the delay bound $\delta$. In all cases, we start from the same initialization and evaluate progress using the full objective as a function of the number of block updates.

To make the comparison fair, we tune the stepsize independently for each method configuration. Concretely, for \algname{PD} we sweep over a fixed grid of learning rates and keep the run with the smallest final full objective. For \algname{RPD}, we repeat the same stepsize sweep separately for each tested value of $\delta$, and again keep the run with the smallest final full objective. The corresponding best-tuned learning rate is reported in the figures.

Figure~\ref{fig:rpd-validation-and-scaling}, the left panel of Figure~\ref{fig:gpd-validation}, compares the best-tuned \algname{PD} trajectory against the best-tuned \algname{RPD} trajectories for several delay bounds $\delta$. The results show a strong dependence of optimization quality on the delay parameter. Small and moderate delay bounds yield substantially better objective decay, whereas larger values of $\delta$ lead to slower convergence and require smaller tuned learning rates in the same delay-sweep experiment.

Together with the fixed-time scaling panel in Figure~\ref{fig:fixed-time-scaling}, this delay sweep highlights a tension between pipeline efficiency and optimization quality in the global-history \algname{RPD} abstraction. On the one hand, \algname{PD} {\rm 1F1B} induces exact delays whose natural scale grows quadratically with the number of stages. On the other hand, \algname{RPD} experiments show that large delay bounds are harmful for optimization and force more conservative stepsizes. Thus, while \algname{PD} is highly attractive from the systems perspective, its induced global-history staleness becomes progressively less favorable as the pipeline depth increases.

\subsection{NanoChat language-modeling experiment}
\label{sec:nanochat-pd-vs-localsgd}

We also compare \algname{PD} and \algname{LocalSGD} on a small autoregressive character-level language-modeling task implemented by the NanoChat-style causal Transformer used in our LLM simulator. The model is partitioned into pipeline stages: the first stage contains token and positional embeddings, the last stage contains the final layer normalization and language-modeling head, and the remaining stages contain Transformer blocks. We train on Tiny Shakespeare and measure last-stage forward cross-entropy loss as a function of simulated wall-clock time.

The model has $11{,}123{,}200$ trainable parameters and is split into $S=16$ pipeline stages. We use sequence length $128$, embedding dimension $256$, $8$ attention heads, mini-batch size $32$, and $1024$ data batches. \algname{PD} is run with $1024$ microbatches, producing a simulated-time budget of $2078$ pipeline ticks and $16{,}384$ backward block updates. \algname{LocalSGD} uses $R=16$ parallel local runs. The selected \algname{LocalSGD} synchronization period is $H=5$, obtained by empirical tuning; the command below reproduces this selected configuration. Additional values of $H$ can be reproduced by rerunning the same command with a different \texttt{--local-steps} value. Batches are shuffled during training.

The comparison uses a matched simulated-time budget. We first generate the \algname{PD} timeline and use its length as the target horizon. For the selected \algname{LocalSGD} setting $H=5$, the smallest matching schedule uses $874$ microbatches and has the same $2078$-tick length. Thus the $x$-axis in \Cref{fig:pd-vs-localsgd-nanochat} is simulated wall-clock time rather than parameter-update count. Under this protocol, the tuned \algname{PD} run obtains the lowest loss.

The experiment was run with:
\begin{verbatim}
python llm_pd_vs_sgd.py \
  --dataset tiny_shakespeare \
  --require-cuda \
  --num-stages 16 \
  --num-data-batches 1024 \
  --batch-size 32 \
  --seq-len 128 \
  --embed-dim 256 \
  --num-heads 8 \
  --pd-num-microbatches 1024 \
  --local-num-runs 16 \
  --local-steps 5 \
  --tune-stepsizes \
  --tuning-seeds 0 \
  --pd-lr-grid pow2:-4:2 \
  --local-sgd-lr-grid pow2:-4:2 \
  --shuffle-batches
\end{verbatim}
The \texttt{pow2:a:b} grid denotes $\{2^a,2^{a+1},\dots,2^{b-1}\}$, so both \algname{PD} and \algname{LocalSGD} are tuned over $\{2^{-4},2^{-3},\dots,2^1\}$ in this run. The script records the selected learning rates, final losses, schedules, curves, and learning-rate sweep summaries in the output directory.

\subsection{Existing Assets and Licenses}
\label{sec:existing-assets}

We use several existing assets in the experimental part of the paper. For the larger LLM experiments, we use a small NanoChat/nanoGPT-style Transformer implementation based on the open-source educational codebases of Karpathy. NanoChat and nanoGPT are released under the MIT License. We cite these repositories and respect their license terms.

For the language-modeling data, we use Tiny Shakespeare, a standard small character-level language-modeling dataset consisting of approximately 40,000 lines of Shakespeare text. We credit the dataset source through Karpathy's \texttt{char-rnn}/Tiny Shakespeare reference. The dataset is used only for small-scale optimization experiments and not as a newly introduced dataset.

No proprietary datasets, scraped private data, user data, or restricted-access pretrained models are used in our experiments.

\clearpage

\section{A stronger convergence theorem under the Polyak--\L ojasiewicz condition}

This section is not used in the main theoretical or empirical comparison. We include it only as a possible direction for future algorithms that use stale-read extrapolation.

We continue to work under the random pair-sampling model of the previous subsection.

\begin{assumption}[Polyak--\L ojasiewicz condition]
\label{ass:gpd_PL}
There exists $\mu>0$ such that
\[
\frac12 \|\nabla f(w)\|^2 \ge \mu\bigl(f(w)-f_\star\bigr)
\qquad \forall w\in\R^d.
\]
\end{assumption}

Under this additional assumption, the \algname{RPD} method enjoys a linear convergence guarantee up to a neighborhood whose size depends on the stepsize and the staleness level.

\begin{theorem}[Linear convergence to a neighborhood for \algname{RPD} under PL]
\label{thm:gpd_random_PL}Let \Cref{ass:gpd_random_pair}, \Cref{ass:gpd_lower_random}, \Cref{ass:gpd_Lsmooth_random}, \Cref{ass:gpd_staleness_random}, \Cref{ass:gpd_bounded_block_random}, \Cref{ass:gpd_PL} hold.
Assume
\[
0<\gamma\le \frac1L.
\]
Then for every $k\ge 0$,
\begin{align}
\E[f(w_{k+1})-f_\star]
&\le
\left(1-\frac{\mu\gamma}{S}\right)\E[f(w_k)-f_\star]
+
\frac{L\gamma^2}{2}G^2
+
\frac{L^2\gamma^3\delta^2}{2S}G^2.
\label{eq:gpd_PL_one_step}
\end{align}
\end{theorem}

The above theorem implies that for every $K\ge 1$,
\begin{align}
\E[f(w_K)-f_\star]
&\le
\left(1-\frac{\mu\gamma}{S}\right)^K \bigl(f(w_0)-f_\star\bigr)
+
\frac{
\frac{L\gamma^2}{2}G^2
+
\frac{L^2\gamma^3\delta^2}{2S}G^2
}{
\mu\gamma/S
}
\notag\\
&=
\left(1-\frac{\mu\gamma}{S}\right)^K \bigl(f(w_0)-f_\star\bigr)
+
\frac{SL\gamma}{2\mu}G^2
+
\frac{L^2\gamma^2\delta^2}{2\mu}G^2.
\label{eq:gpd_PL_linear}
\end{align}

\begin{proof}
From the proof of \Cref{thm:rpd-random-nonconvex-app}, we already have
\begin{align}
\E[f(w_{k+1})\mid\mathcal F_k]
&\le
f(w_k)
-
\frac{\gamma}{2S}\|\nabla f(w_k)\|^2
+
\frac{\gamma}{2S}\|\nabla f(w_k)-\nabla f(z_k)\|^2
+
\frac{L\gamma^2}{2}G^2.
\label{eq:gpd_PL_start}
\end{align}
By $L$-smoothness and \Cref{lem:rpd-stale-distance},
\[
\|\nabla f(w_k)-\nabla f(z_k)\|
\le
L\|w_k-z_k\|
\le
L\gamma\delta G,
\]
and therefore
\[
\|\nabla f(w_k)-\nabla f(z_k)\|^2
\le
L^2\gamma^2\delta^2G^2.
\]
Substituting this into \eqref{eq:gpd_PL_start}, we get
\[
\E[f(w_{k+1})\mid\mathcal F_k]
\le
f(w_k)
-
\frac{\gamma}{2S}\|\nabla f(w_k)\|^2
+
\frac{L^2\gamma^3\delta^2}{2S}G^2
+
\frac{L\gamma^2}{2}G^2.
\]
Now apply the PL inequality:
\[
\frac12\|\nabla f(w_k)\|^2 \ge \mu\bigl(f(w_k)-f_\star\bigr),
\]
which implies
\[
\frac{\gamma}{2S}\|\nabla f(w_k)\|^2
\ge
\frac{\mu\gamma}{S}\bigl(f(w_k)-f_\star\bigr).
\]
Hence
\[
\E[f(w_{k+1})\mid\mathcal F_k]
\le
f(w_k)
-
\frac{\mu\gamma}{S}\bigl(f(w_k)-f_\star\bigr)
+
\frac{L\gamma^2}{2}G^2
+
\frac{L^2\gamma^3\delta^2}{2S}G^2.
\]
Subtract $f_\star$ from both sides:
\[
\E[f(w_{k+1})-f_\star\mid\mathcal F_k]
\le
\left(1-\frac{\mu\gamma}{S}\right)\bigl(f(w_k)-f_\star\bigr)
+
\frac{L\gamma^2}{2}G^2
+
\frac{L^2\gamma^3\delta^2}{2S}G^2.
\]
Taking total expectation proves \eqref{eq:gpd_PL_one_step}.

Now define
\[
\Delta_k \eqdef \E[f(w_k)-f_\star],
\qquad
a \eqdef 1-\frac{\mu\gamma}{S},
\qquad
b \eqdef \frac{L\gamma^2}{2}G^2+\frac{L^2\gamma^3\delta^2}{2S}G^2.
\]
Then
\[
\Delta_{k+1}\le a\Delta_k+b.
\]
Iterating this recursion gives
\[
\Delta_K \le a^K \Delta_0 + b\sum_{t=0}^{K-1} a^t
\le
a^K \Delta_0 + \frac{b}{1-a}.
\]
Since $1-a=\mu\gamma/S$, we obtain
\[
\Delta_K
\le
\left(1-\frac{\mu\gamma}{S}\right)^K (f(w_0)-f_\star)
+
\frac{
\frac{L\gamma^2}{2}G^2+\frac{L^2\gamma^3\delta^2}{2S}G^2
}{
\mu\gamma/S
},
\]
which simplifies to \eqref{eq:gpd_PL_linear}.
\end{proof}

\clearpage
\section{Exploratory Nesterov-inspired extrapolated variant of Randomized PipeDream (NAG-RPD)}

We now describe a predictive variant of \algname{RPD} in which the stale model is extrapolated before the gradient is evaluated. This appendix is exploratory and is not used in the main comparison. The worst-case bound below does not show that the extrapolation parameter improves the guarantee; it only records what follows from the same smoothness argument as the base \algname{RPD} theorem.

\begin{algorithm}[H]
\caption{Nesterov-inspired extrapolated variant of Randomized PipeDream (\algname{NAG-RPD})}
\label{alg:nag_gpd}
\begin{algorithmic}[1]
\STATE Input: stepsize $\gamma>0$, prediction parameter $\alpha\ge 0$, initial point $w_0\in\R^d$
\STATE Initialize $w_{-1}\eqdef w_0$
\FOR{$k=0,1,2,\dots$}
    \STATE Select a delay vector $j_k$ from the past, with $0\le k-j_k^{(s)}\le \delta$ for all $s$
    \STATE Read the stale model
    \[
    z_k \eqdef \bigl(w_{j_k^{(1)}}^{(1)},\dots,w_{j_k^{(S)}}^{(S)}\bigr)
    \]
    \STATE Form the extrapolated stale model
    \[
    \widehat z_k
    \eqdef
    \bigl(
    \widehat z_k^{(1)},\dots,\widehat z_k^{(S)}
    \bigr),
    \qquad
    \widehat z_k^{(s)}
    \eqdef
    w_{j_k^{(s)}}^{(s)}
    +
    \alpha\bigl(
    w_{j_k^{(s)}}^{(s)}-w_{j_k^{(s)}-1}^{(s)}
    \bigr),
    \]
    with the convention $w_{-1}\eqdef w_0$
    \STATE Sample $(s_k,m_k)$ uniformly from $\{1,\dots,S\}\times\{1,\dots,M\}$, conditionally independently of $\widehat z_k$
    \STATE Update the active stage block:
    \[
    w_{k+1}^{(s_k)}
    =
    w_k^{(s_k)}-\gamma\,\nabla_{w^{(s_k)}} f_{m_k}(\widehat z_k)
    \]
    \STATE Keep the remaining blocks unchanged:
    \[
    w_{k+1}^{(r)}=w_k^{(r)}
    \qquad \forall r\neq s_k
    \]
\ENDFOR
\end{algorithmic}
\end{algorithm}

The increment
\(
w_{j_k^{(s)}}^{(s)}-w_{j_k^{(s)}-1}^{(s)}
\)
is the previous \emph{global-history} increment in block $s$ and may be zero if
the previous global update changed another block. A version that extrapolates
from the previous local update of block $s$ would require different notation
and a separate analysis.

\begin{lemma}[Unbiasedness of the \algname{NAG-RPD} update]
\label{lem:nag_gpd_unbiased_direction}
Let \Cref{ass:gpd_random_pair} hold, and define
\[
\widehat g_k
\eqdef
U_{s_k}\nabla_{w^{(s_k)}} f_{m_k}(\widehat z_k).
\]
Then
\[
\Exp{S\widehat g_k\mid \mathcal F_k}
=
\nabla f(\widehat z_k),
\]
where $\mathcal F_k$ includes the current delay vector and extrapolated stale point, but not the current sampled pair $(s_k,m_k)$.
\end{lemma}

\begin{proof}
Conditioned on $\mathcal F_k$, the extrapolated point $\widehat z_k$ is fixed. Hence
\begin{align*}
\Exp{\widehat g_k\mid \mathcal F_k}
&=
\frac{1}{SM}
\sum_{s=1}^S\sum_{m=1}^M
U_s \nabla_{w^{(s)}} f_m(\widehat z_k) \\
&=
\frac{1}{S}
\sum_{s=1}^S
U_s \left( \frac1M\sum_{m=1}^M \nabla_{w^{(s)}} f_m(\widehat z_k) \right) \\
&=
\frac1S
\sum_{s=1}^S
U_s \nabla_{w^{(s)}} f(\widehat z_k)
=
\frac1S \nabla f(\widehat z_k).
\end{align*}
\end{proof}

\begin{lemma}[Distance between current and extrapolated stale models]
\label{lem:nag-rpd-stale-extrapolated-distance}
Let \Cref{ass:gpd_staleness_random} hold, and assume that all block gradients evaluated by the \algname{NAG-RPD} trajectory are bounded by $G$.
Then for every $k\ge 0$,
\[
\|w_k-\widehat z_k\|
\le
\gamma(\delta+\alpha\sqrt S)G.
\]
\end{lemma}

\begin{proof}
Fix $k\ge 0$.
By definition,
\[
\widehat z_k^{(s)}
=
w_{j_k^{(s)}}^{(s)}
+
\alpha\bigl(w_{j_k^{(s)}}^{(s)}-w_{j_k^{(s)}-1}^{(s)}\bigr).
\]
Hence
\[
w_k-\widehat z_k
=
(w_k-z_k)-\alpha\,u_k,
\]
where
\[
u_k
\eqdef
\bigl(
w_{j_k^{(1)}}^{(1)}-w_{j_k^{(1)}-1}^{(1)},
\dots,
w_{j_k^{(S)}}^{(S)}-w_{j_k^{(S)}-1}^{(S)}
\bigr).
\]
By the triangle inequality,
\[
\|w_k-\widehat z_k\|
\le
\|w_k-z_k\|+\alpha\|u_k\|.
\]

The same argument as in \Cref{lem:rpd-stale-distance} gives
\[
\|w_k-z_k\|\le \gamma\delta G,
\]
because every update changes a single block by norm at most $\gamma G$.

It remains to bound $\|u_k\|$.
At any iteration $t$, only one block is updated, and
\[
\|w_{t+1}-w_t\|
=
\gamma\|\nabla_{w^{(s_t)}} f_{m_t}(\cdot)\|
\le
\gamma G.
\]
For each stage $s$, the vector
\[
w_{j_k^{(s)}}^{(s)}-w_{j_k^{(s)}-1}^{(s)}
\]
is exactly the increment of block $s$ at iteration $j_k^{(s)}-1$, or zero if that block did not change there.
Thus $u_k$ is obtained by collecting at most one block increment for each block.
Each block increment has norm at most $\gamma G$, so the concatenated vector satisfies
\[
\|u_k\|\le \gamma G\sqrt S.
\]
Therefore
\[
\|w_k-\widehat z_k\|
\le
\gamma\delta G+\alpha\gamma G\sqrt S
=
\gamma(\delta+\alpha\sqrt S)G.
\]
\end{proof}

\begin{theorem}[Nonconvex convergence of \algname{NAG-RPD}]
\label{thm:nag_gpd_random_nonconvex}
Let \Cref{ass:gpd_random_pair}, \Cref{ass:gpd_lower_random},
\Cref{ass:gpd_Lsmooth_random}, \Cref{ass:gpd_staleness_random},
and \Cref{ass:gpd_bounded_block_random} hold.
Assume also that the trajectory-bounded gradient condition holds at the extrapolated stale points $\widehat z_k$.
Assume
\[
0<\gamma\le \frac1L,
\qquad
\alpha\ge 0.
\]
Then for every integer $K\ge 1$,
\begin{align}
\frac1K\sum_{k=0}^{K-1}\Exp{\|\nabla f(w_k)\|^2}
&\le
\frac{2S\bigl(f(w_0)-f_\star\bigr)}{\gamma K}
+
\gamma SL G^2
+
\gamma^2L^2(\delta+\alpha\sqrt S)^2 G^2.
\label{eq:nag_gpd_random_nonconvex_rate}
\end{align}
\end{theorem}

\begin{proof}
Define
\[
\widehat g_k
\eqdef
U_{s_k}\nabla_{w^{(s_k)}} f_{m_k}(\widehat z_k).
\]
Then the update is
\[
w_{k+1}=w_k-\gamma \widehat g_k.
\]
By $L$-smoothness of $f$,
\[
f(w_{k+1})
\le
f(w_k)
-
\gamma\inner{\nabla f(w_k)}{\widehat g_k}
+
\frac{\gamma^2L}{2}\|\widehat g_k\|^2.
\]
Taking conditional expectation and using \Cref{lem:nag_gpd_unbiased_direction},
\begin{equation}
\Exp{f(w_{k+1})\mid\mathcal F_k}
\le
f(w_k)
-
\frac{\gamma}{S}\inner{\nabla f(w_k)}{\nabla f(\widehat z_k)}
+
\frac{\gamma^2L}{2}\Exp{\|\widehat g_k\|^2\mid\mathcal F_k}.
\label{eq:nag_gpd_conditional}
\end{equation}
By the trajectory-bounded gradient assumption at the extrapolated stale points,
\[
\|\widehat g_k\|
=
\|\nabla_{w^{(s_k)}} f_{m_k}(\widehat z_k)\|
\le G,
\]
so
\[
\Exp{\|\widehat g_k\|^2\mid\mathcal F_k}\le G^2.
\]
Substituting into \eqref{eq:nag_gpd_conditional},
\[
\Exp{f(w_{k+1})\mid\mathcal F_k}
\le
f(w_k)
-
\frac{\gamma}{S}\inner{\nabla f(w_k)}{\nabla f(\widehat z_k)}
+
\frac{\gamma^2LG^2}{2}.
\]

Now use
\[
\inner{a}{b}
\ge
\frac12\|a\|^2-\frac12\|a-b\|^2
\]
with
\[
a=\nabla f(w_k),
\qquad
b=\nabla f(\widehat z_k).
\]
This gives
\[
-\frac{\gamma}{S}\inner{\nabla f(w_k)}{\nabla f(\widehat z_k)}
\le
-\frac{\gamma}{2S}\|\nabla f(w_k)\|^2
+
\frac{\gamma}{2S}\|\nabla f(w_k)-\nabla f(\widehat z_k)\|^2.
\]
Hence
\[
\Exp{f(w_{k+1})\mid\mathcal F_k}
\le
f(w_k)
-
\frac{\gamma}{2S}\|\nabla f(w_k)\|^2
+
\frac{\gamma}{2S}\|\nabla f(w_k)-\nabla f(\widehat z_k)\|^2
+
\frac{\gamma^2LG^2}{2}.
\]

By $L$-smoothness of $f$,
\[
\|\nabla f(w_k)-\nabla f(\widehat z_k)\|
\le
L\|w_k-\widehat z_k\|.
\]
Using \Cref{lem:nag-rpd-stale-extrapolated-distance},
\[
\|w_k-\widehat z_k\|\le \gamma(\delta+\alpha\sqrt S)G,
\]
and therefore
\[
\|\nabla f(w_k)-\nabla f(\widehat z_k)\|^2
\le
\gamma^2L^2(\delta+\alpha\sqrt S)^2G^2.
\]
Substituting this gives
\[
\Exp{f(w_{k+1})\mid\mathcal F_k}
\le
f(w_k)
-
\frac{\gamma}{2S}\|\nabla f(w_k)\|^2
+
\frac{\gamma^3L^2(\delta+\alpha\sqrt S)^2G^2}{2S}
+
\frac{\gamma^2LG^2}{2}.
\]
Taking expectation and summing from $k=0$ to $K-1$,
\begin{align*}
\Exp{f(w_K)}
&\le
f(w_0)
-
\frac{\gamma}{2S}\sum_{k=0}^{K-1}\Exp{\|\nabla f(w_k)\|^2}
+
K\frac{\gamma^3L^2(\delta+\alpha\sqrt S)^2G^2}{2S}
+
K\frac{\gamma^2LG^2}{2}.
\end{align*}
Using $\Exp{f(w_K)}\ge f_\star$, we get
\[
\frac{\gamma}{2S}\sum_{k=0}^{K-1}\Exp{\|\nabla f(w_k)\|^2}
\le
f(w_0)-f_\star
+
K\frac{\gamma^3L^2(\delta+\alpha\sqrt S)^2G^2}{2S}
+
K\frac{\gamma^2LG^2}{2}.
\]
Multiply by $2S/(\gamma K)$ to obtain
\[
\frac1K\sum_{k=0}^{K-1}\Exp{\|\nabla f(w_k)\|^2}
\le
\frac{2S\bigl(f(w_0)-f_\star\bigr)}{\gamma K}
+
\gamma SLG^2
+
\gamma^2L^2(\delta+\alpha\sqrt S)^2G^2
,
\]
which proves \eqref{eq:nag_gpd_random_nonconvex_rate}.
\end{proof}


\end{document}